\documentclass[10pt,twocolumn,letterpaper,table]{article}

\usepackage[pagenumbers]{iccv} 

\usepackage{algorithm}
\usepackage[noend]{algpseudocode}
\usepackage{tcolorbox}
\usepackage{titletoc}
\usepackage{times}
\usepackage{epsfig}
\usepackage{graphicx}
\usepackage{amsmath}
\usepackage{amsthm}
\usepackage{amssymb}
\usepackage[symbol]{footmisc}
\usepackage{float}
\usepackage{tabularx}
\usepackage{booktabs}
\theoremstyle{definition}
\newtheorem{definition}{Definition}

\newtheorem{lemma}{Lemma}

\newtheorem{condition}{Condition}

\usepackage{lineno}

\usepackage{listings}
 \usepackage{xcolor}

\usepackage[T1]{fontenc}      
\usepackage[utf8]{inputenc}   
\usepackage{lmodern}          
\usepackage{microtype}        

\usepackage{mathptmx}         

\usepackage[T1]{fontenc}
\usepackage[utf8]{inputenc}
\usepackage{lmodern}
\usepackage{microtype}
\usepackage{mathptmx}
\usepackage{graphicx} 
\usepackage{amsmath,amssymb}
\usepackage{pgfplots}
\pgfplotsset{compat=1.18}

\definecolor{codegreen}{rgb}{0,0.6,0}
\definecolor{codegray}{rgb}{0.5,0.5,0.5}
\definecolor{codepurple}{rgb}{0.58,0,0.82}
\definecolor{backcolour}{rgb}{0.95,0.95,0.92}
\lstdefinestyle{mystyle}{
    backgroundcolor=\color{backcolour},
    commentstyle=\color{codegreen},
    keywordstyle=\color{magenta},
    numberstyle=\tiny\color{codegray},
    stringstyle=\color{codepurple},
    basicstyle=\ttfamily\footnotesize,
    breakatwhitespace=false,
    breaklines=true,
    captionpos=b,
    keepspaces=true,
    numbers=left,
    numbersep=5pt,
    showspaces=false,
    showstringspaces=false,
    showtabs=false,
    tabsize=2
}
\definecolor{iccvblue}{rgb}{0.21,0.49,0.74}
\usepackage[pagebackref,breaklinks,colorlinks,allcolors=iccvblue]{hyperref}

\def\paperID{4310} 
\def\confName{ICCV}
\def\confYear{2025}

\title{Is CLIP ideal? No. Can we fix it? Yes!}
\author{Raphi Kang \quad Yue Song \quad  Gerogia Gkioxari \quad  Pietro Perona\\
California Institute of Technology\\
Pasadena, CA\\
{\tt\small rkang@caltech.edu}
}

\begin{document}
\maketitle
\setlength{\abovedisplayskip}{1pt}
\setlength{\belowdisplayskip}{1pt}

\begin{abstract}
Contrastive Language-Image Pre-Training (CLIP) is a popular method for learning multimodal latent spaces with well-organized semantics.
Despite its wide range of applications, CLIP's latent space is known to fail at handling complex visual-textual interactions. 
Recent works attempt to address its shortcomings with data-centric or algorithmic approaches. But what if the problem is more fundamental, and lies in the geometry of CLIP?
Toward this end, we rigorously analyze CLIP's latent space properties, and prove that no CLIP-like joint embedding space exists which can correctly do any two of the following at the same time: 1. represent basic descriptions and image content, 2. represent attribute binding, 3. represent spatial location and relationships, 4. represent negation. 
Informed by this analysis, we propose Dense Cosine Similarity Maps (DCSMs) as a principled and interpretable scoring method for CLIP-like models, which solves the fundamental limitations of CLIP by retaining the semantic topology of the image patches and text tokens. This method improves upon the performance of classical CLIP-like joint encoder models on a wide array of benchmarks. We share our code and data here for reproducibility: \url{https://github.com/Raphoo/DCSM_Ideal_CLIP}
\end{abstract}

\section{Introduction}

Contrastive Language-Image Pre-Training (CLIP) is a widely used method for pretraining vision language model (VLM) embeddings in downstream applications~\cite{radford_learning_2021}. It jointly trains an image and text encoder, mapping both modalities into a shared latent space. In this space, cosine similarity between embeddings reflects semantic similarity between the text and image: a high value indicates a semantic match between text and image, while a low value suggests that they are unrelated.

CLIP is simple, computationally efficient, and CLIP-based zero-shot multimodal tasks such as image classification and text-based retrieval can be impressively accurate. Unlike autoregressive models where an increasingly large number of images and texts to compare requires prohibitively longer context windows, CLIP can process each image and text prompt separately and only requires a simple inner product calculation for scoring. As a result, the list of systems that use CLIP embeddings toward more complex capabilities, such as image captioning \cite{BLIP_li}, visual question answering \cite{shen_how_2022}, and text-guided image manipulation/generation \cite{kim_diffusionclip_2022,nichol_glide_2022}, gets longer by the day.

\begin{figure}[t]
    \centering
\includegraphics[width=0.99\linewidth]{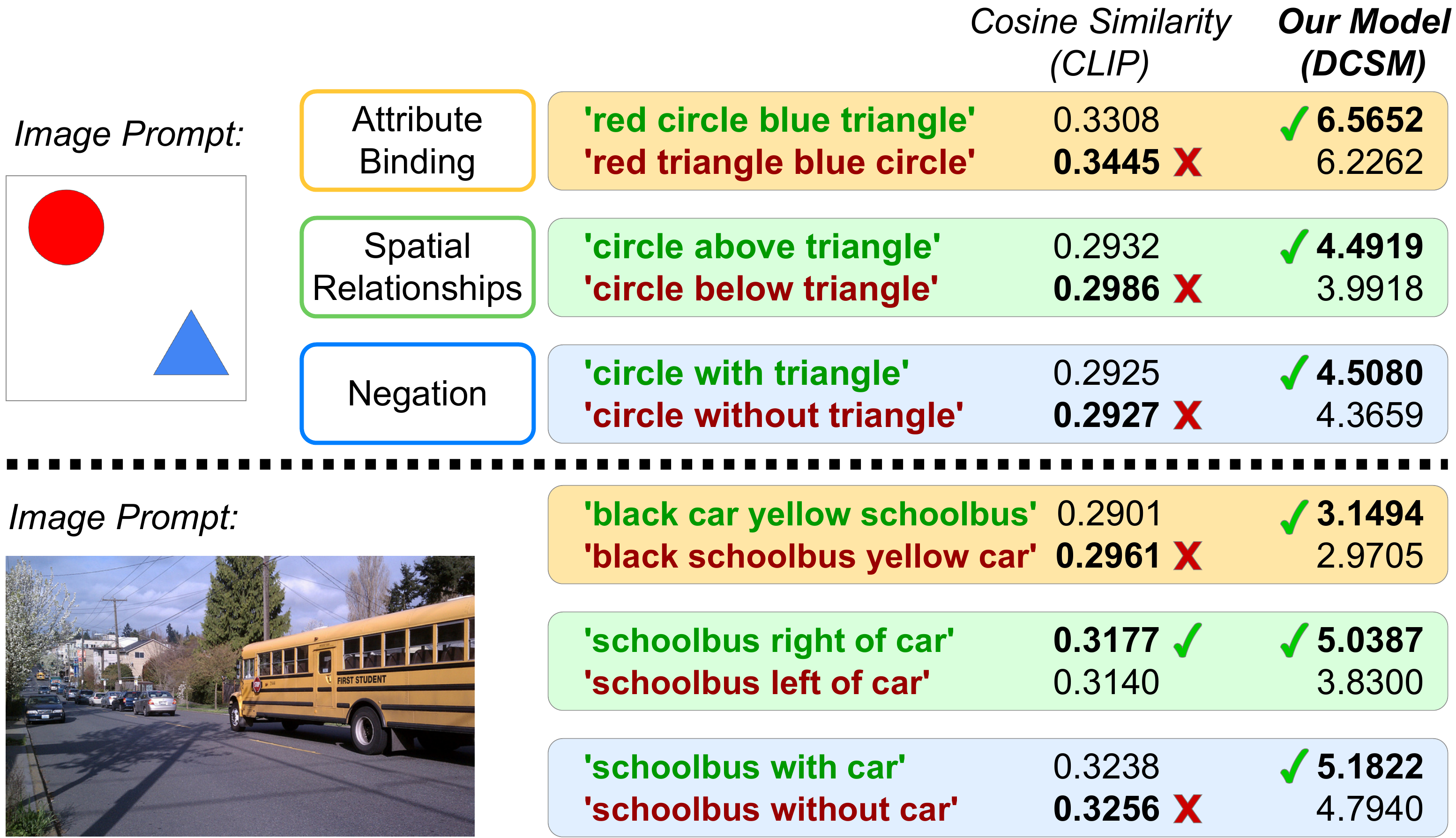}
\vspace{-3mm}
    \caption{\textbf{CLIP scores do not accurately reflect semantics of text prompts} due to inherent geometric limitations. For five out of six pairs of captions, the incorrect text has the higher CLIP score with the image. By contrast, our proposed Dense Cosine Similarity Maps (DCSM) correctly scores matched pairs. The similarity score is unnormalized because it is predicted by a neural network. CLIP scores computed with OpenAI-CLIP ViT-B/32.} 
\label{fig:simple_cond_show}
\vspace{-7mm}
\end{figure}

CLIP, however, is not perfect.  It struggles with spatial reasoning and compositional understanding  \cite{yuksekgonul_when_2023,kamath_whats_2023,tong_eyes_2024}, concept and attribute binding \cite{lewis_does_2024,newman_pre-trained_2024,campbell_understanding_2024}, as well as negation \cite{alhamoud_vision-language_2025,singh_learn_2024}, to name a few shortcomings. Examples in Fig.~\ref{fig:simple_cond_show} demonstrate these failure modes.
These defects of CLIP impact downstream models and tasks. For example, search engines which use CLIP will bring up images of yellow coats to the prompt ``not a yellow coat", and CLIP-based generative models cannot reflect spatial relationships in the text prompt.
Ever since CLIP's conception, the vision-language community has worked to improve its semantics via adjustments in the training data~\cite{yuksekgonul_when_2023,wang_boosting_2024,singh_learn_2024,alhamoud_vision-language_2025,paiss_teaching_2023,radenovic_filtering_2023,koohpayegani_genie_2024,wang_advancing_2024,li_laion-sg_2024} or the architecture and training procedure~\cite{li_covlm_2023,li_interpretable_2024,shao_explore_2024,chen_spatialvlm_2024,wang_sclip_2023,menon_visual_2022, zhai_sigmoid_2023}.
But could we be pushing Sisyphus's boulder? 
In this work, we suggest to take a step back and reassess CLIP's basic geometry and philosophy from first principles.


We ask two questions: (1) \textit{Is it possible for there to exist a CLIP-like latent space which has all the properties necessary to succeed in popular VLM tasks, using cosine similarity as the reference metric?}
(2) \textit{If no such space exists - is there a way to ``rescue" this latent space without forsaking CLIP altogether?}

To address these queries, we first formalize the geometry of the CLIP latent space and establish a list of conditions which must be satisfied in order for it to understand the precise meaning of images and texts, as required by popular VLM benchmarks \cite{tong_eyes_2024, yuksekgonul_when_2023, kamath_whats_2023, lewis_does_2024, alhamoud_vision-language_2025, hsieh_sugarcrepe_2023}. Second, we show that these conditions are\textit{ not achievable} when using cosine similarity on the unit hypersphere. This answers the first question: an ideal CLIP space \textit{cannot} exist. Finally, we propose a lightweight downstream CNN which utilizes a topological map relying on cosine similarity scores between each text token and image patch embedding, to produce a more perceptive text-image distance score in lieu of naive cosine similarity in CLIP-space. 


In summary, our contributions are:

\begin{itemize}
    \item  \textbf{Problem Identification}: We find that naive cosine similarity on unit vector embeddings have fundamental geometric restrictions preventing it from accurately representing (1) attribute binding to distinct object concepts, (2) spatial relationships and location, and (3) negation.
    
    \item \textbf{Analysis/Proof}: We define the CLIP latent space as a projection of images and texts which are composed of atomic objects, attributes, and spatial relationships, and formalize desired conditions for the latent space in order for it to succeed on existing VLM benchmarks. We prove for each condition that no satisfactory vector space exists. 
    
    \item \textbf{Topology as a Solution}: We propose to use a Dense Cosine Similarity Map (DCSM) from pre-trained CLIP to produce a more comprehensive text-image score. We present a prototype with a simple two-layer CNN module. 
    
    \item \textbf{Experimental evaluation}: We evaluate our method on multiple benchmark regimes against CLIP-like models, and observe consistent performance gain across tasks. 
\end{itemize}
\section{Prior Work}
\label{sec:prior-work}



\noindent\textbf{Vision Language Models.} Building upon the principles of large-scale contrastive pretraining and joint representation learning, recent CLIP-like VLMs~\cite{jia2021scaling,yaofilip2022,yaofilip2022,mu2022slip,li2021align,yu2022coca} have significantly bridged the gap between vision and language. Beyond CLIP models, autoregressive VLMs ~\cite{liu2024improved,liu2024visual} have emerged as compelling alternatives by jointly attending to both text and image embeddings. Neurosymbolic program synthesis methods~\cite{surís2023vipergptvisualinferencepython,gupta2022visualprogrammingcompositionalvisual,marsili2025visualagenticaispatial} also mitigate some of CLIP's limitations by formulating text-image semantic distance acquisition as several sub-problems. While these methods are more comprehensive than CLIP in complex visual reasoning, most of them rely on a latent space of CLIP-like models as a core submodule. As such, there remains a strong motivation to continue improving CLIP architectures by addressing their inherent shortcomings.

\noindent\textbf{Empirical Limitations of CLIP.} A growing body of work has revealed several limitations of CLIP in handling complex visual-text interactions. Specifically, CLIP struggles with accurately binding attributes to concepts in multi-object scenes~\cite{lewis_does_2024, newman_pre-trained_2024, campbell_understanding_2024}, exhibits misinterpretations of object layouts or conflates multiple entities within a single scene~\cite{yuksekgonul_when_2023, lewis_does_2024,newman_pre-trained_2024}, and has difficulty in accurately representing negation~\cite{alhamoud_vision-language_2025,singh_learn_2024}. To systematically evaluate these limitations, various vision language benchmarks have been proposed targeting different tasks, including attributes, relationships, and order (ARO)~\cite{yuksekgonul_when_2023}, Sugarcrepe~\cite{hsieh_sugarcrepe_2023}, VL-checklist~\cite{zhao_vl-checklist_2023}, WhatsUp \cite{kamath_whats_2023}, Multimodal Visual Patterns (MMVP)~\cite{tong2024eyes}, and NegBench~\cite{alhamoud_vision-language_2025}. 

\begin{figure*}[t]
    \centering
    \includegraphics[width=0.95\linewidth]{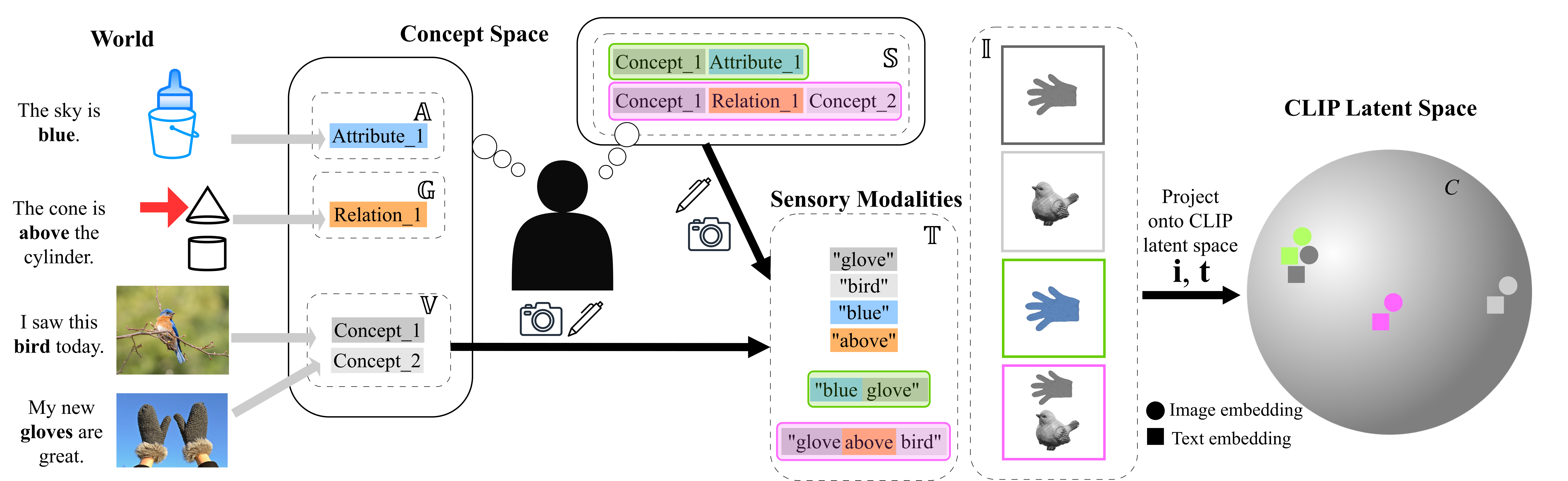}
    \vspace{-3mm}
    \caption{\textbf{Graphical illustration of defined concept sets}.
    Humans can parse visual stimuli from the real world and organize them into Object Concepts $\mathbb{V}$, Attributes $\mathbb{A}$ which adorn objects, and Relationships $\mathbb{G}$ between objects. These concepts can be ordered into Composed Scenes $\mathbb{S}$. Here, $\mathbb{V},\mathbb{A}, \mathbb{G}, \text{ and }\mathbb{S}$ are strict subsets of the set of all real world concepts.
    We can communicate these composed concepts via language $\mathbb{T}$, or by taking pictures of exemplars $\mathbb{I}$. These sentences and images can be projected onto the CLIP latent space \textit{C} via a text encoder \textbf{t} or an image encoder \textbf{i}. Elements in distinct sets with the same color have a one-to-one correspondence.}
    \label{fig:defn_explain}
    \vspace{-5mm}
\end{figure*}

\noindent\textbf{Geometric Analysis of CLIP.} 
Several studies have examined the geometric properties of CLIP. Some works empirically assess the CLIP latent space and find that existing CLIP models underutilize the latent space by having a modality gap between text and image embeddings \cite{liang_mind_2022}, being highly anisotropic even within the same modality~\cite{levi_double-ellipsoid_2024},
and that CLIP training is empirically unstable~\cite{steck_is_2024}. Beyond these empirical studies, one work rigorously defines CLIP as a mapping from images and texts to a latent unit hypersphere~\cite{brody_potential_2023}.
Our work extends this perspective.

Other works propose alternative geometries/similarity measures for CLIP. For instance, \cite{desai2023hyperbolic} and
\cite{pal2025compositional} present hyperbolic CLIP using Lorentzian distance.

These show
great potential, but as they do not aim to offer formal logical analysis of geometries/metrics, there is no conclusive
evidence that they resolve the fundamental ambiguity in
CLIP. 

Extended related work is provided in the Supplements.






\section{The Ideal CLIP}

The goal of CLIP is ``to efficiently learn visual concepts from natural language supervision''~\cite{radford_learning_2021}. 
What are {\em visual concepts}? We first establish formal notations with which to discuss visual concepts in the latent space. Once we have the tools to discuss CLIP mathematically, we will formalize the properties necessary to make CLIP ideal.

Fig.~\ref{fig:defn_explain} illustrates how we think about the relationship between the world, its representation in images and language, and the eventual projection of images and text onto CLIP space.
In this view, the ``gist'' of scene perception is broken down into categorization of objects, their salient attributes, and the spatial layout of the scene, as proposed by the cognitive psychology literature \cite{Biederman1987-sg,rosch_principles_1988,li2002rapid,Oliva2005,fei2007we,Wolfe2017}. Theories of perceptual symbol systems \cite{Barsalou1999-ze, Barsalou2003-hl, 9a90845cf49e42f99e597f54212ae629} propose that object categories and spatial relationships are atomic symbols, combinatorially and recursively combined to create scene representations. The Dual Coding theory of cognition \cite{Paivo1981-hl, Paivio1990, Paivio2013} proposes that these composed concepts can be represented as language, imagery, or both. 

In the following we formalize these intuitions with precise definitions, with the goal of analyzing CLIP properties.

\subsection{Concepts, attributes, and compositions}
\label{sec-definitions}

This section sets up a simplified world, which is a strict subset of the real world, within which we will explore whether it is possible to have an {\it ideal } CLIP. Our goal is to show that this is not possible and therefore, {\it a fortiori}, an ideal CLIP is not possible in more general worlds.

\begin{definition} \textbf{(CLIP Embedding Space $C$).}
CLIP embeddings are unit vectors which occupy an $N$-dimensional unit hypersphere. 
Let $\mathbb{I}$ denote any set of images and let $\mathbb{T}$ any set of texts. Then a CLIP-like model is a pair of functions (\textbf{i, t}) where for some Euclidean vector space $C$, consisting of $N$-dimensional unit vectors, we have that \textbf{i} : $\mathbb{I}$ → $C$ and \textbf{t} : $\mathbb{T}$→ $C$. That is, \textbf{i} and \textbf{t} are injective functions which map images and text descriptions into $C \in S^{N-1}$ where $S^{N-1}$ denotes the unit-($N$-1) sphere: ${S^{N-1}=\left\{\textbf{i}(x), \textbf{t}(x)\in \mathbb{R}^{N}:\left\|\textbf{i}(x)\right\|_2=\left\|\textbf{t}(x)\right\|_2=1\right\}}$. 

\end{definition}

\begin{definition} \textbf{(Atomic Concept Set $\mathbb{V}$).}
In the real world, there exist $\gg N$ object concepts. These have hierarchy, and can be homonyms, polysemous, heteronyms, or synonyms. Let $\mathbb{V}$ be a small subset of this real world object concept set, where every $x$ in $\mathbb{V}$ has a distinct visual and linguistic counterpart, and all $x \in \mathbb{V}$ are semantically mutually exclusive, \emph{e.g.} no $x$ is a subclass of another. We select $\mathbb{V}$ so that it contains fewer than $N$ objects: $|\mathbb{V}| = M \leq N$.

Each concept, \emph{e.g.} "apple", corresponds to many physical objects, and each object may be portrayed in many different images. Now consider one representative element, one image, for each set of images corresponding to a concept, and one representative element for each set of texts corresponding to a concept. In this way, we are defining injective functions $f_{V,I} $: $ \mathbb{V} \rightarrow \mathbb{I}$ and $f_{V,T} $: $  \mathbb{V} \rightarrow \mathbb{T}$ which act upon every element in $\mathbb{V}$. For the representative image which visually represents concept $x \in \mathbb{V}$, let its CLIP embedding be: $\textbf{i}(f_{V,I}(x))$, which unless otherwise specified we will denote as $\textbf{i}(x) \in  S^{N-1}$. Similarly we denote some unique text representation of $x$ as $\textbf{t}(f_{V,T}(x)) = \textbf{t}(x) \in S^{N-1}$. 
 Justifications for assumptions are provided in the supplements.
 


\end{definition}

\begin{definition} \textbf{(Attribute Representation Set $\mathbb{A}$).}
In the real world, there exist $\gg N $ attribute concepts. An attribute is not itself an object, but rather a property of the object, like color, size, and material. 
Then let $\mathbb{A}$ be a subset of this large attribute set, where every $a$ in $\mathbb{A}$ is a discrete and visually distinct attribute, and can be applied to every instance of $x$ in $\mathbb{V}$.
For every $a \in \mathbb{A}$, $a$ has one equivalent item in $\mathbb{T}$, such that there exists some injective function $f_{A,T} $: $ \mathbb{A} \rightarrow \mathbb{T}$. 

Let $\textbf{i}(x_a) $ denote an image embedding where concept $x$ has attribute $a$. 
 For any given object, some attributes are highly characteristic (\emph{e.g.,}``red" is a defining feature for an ``apple"), some are relatively neutral (\emph{e.g.,} ``red" for a ``car" may not be as informative), and some are unlikely to be associated with that object (\emph{e.g.,} ``red" for a ``raccoon").
 The representative embedding $\textbf{i}(x)$ $\forall x \in \mathbb{V}$ implicitly portrays attributes that are most likely or typical for that concept. 

\end{definition}

\begin{definition} \textbf{(Compositional Concept Set $\mathbb{G}$)}
In the real world, there exist $\gg N$ compositional concepts, \emph{i.e.,} concepts which describe the composition of a scene. 
Then let $\mathbb{G}$ be a subset of the real compositional concept set, consisting only of elements which describe the absolute location of one object (denoted $g^{<loc>}$) or the spatial relationship between two objects (denoted $g^{<rel>}$).
Each element in $\mathbb{G}$ has a unique equivalent text description in $\mathbb{T}$. That is, there exists an injective function $f_{G,T} $: $ \mathbb{G} \rightarrow \mathbb{T}$. 

For example, for the compositional concept of one object being above another object, we use the notation $g^{<rel>}_{above}$, which has the text representation  $\text{``above"}$.
(In the next definition we explain our notation for using compositional concepts in context with other concepts.) 
When two objects $x$ and $y$ appear in the same image, a spatial relationship is there between them. If one wishes to marginalize this aspect out one may use the notation $\textbf{i}(x,y)$ 
to denote the mean embedding of images depicting concepts $x$ and $y$, averaged over all spatial arrangements of  $x$ and $y$ in the fixed frame.


\end{definition}
\label{def-combinatorial}

\begin{definition} \textbf{(Combinatorial Clause Set 
$\mathbb{S}$) } 
In the real world, object, attribute, and compositional concepts can be organized to create semantic \textit{clauses}. These are a combination of concepts which add context to a unified scene.
Let $\mathbb{S^*}$ be a set of all finite ordered combinations of elements from $\mathbb{V}$, $\mathbb{A}$, and $\mathbb{G}$.  It is a union of all \textit{n}-fold Cartesian products of the full outer joined set of the three. 
$$\mathbb{S^*} = \bigcup_{n=1}^{\infty} \bigl(\mathbb{A} \cup \mathbb{G} \cup \mathbb{V}\bigr)^n$$
Every element in $\mathbb{S^*}$ has an equivalent text description in $\mathbb{T}$. Now, let $\mathbb{S}$ be a subset of elements from $\mathbb{S}^*$ whose textual counterpart is grammatically correct, such that there exist equivalent and distinct mappings between every item in $\mathbb{S}$ and $\mathbb{I}$,$\mathbb{T}$. This means there exist injective functions $f_{S,I} $: $ \mathbb{S} \rightarrow \mathbb{I}$ and $f_{S,T} $: $  \mathbb{S} \rightarrow \mathbb{T}$. For example, the joined clause of the object concept of a glove and the attribute concept of the color blue, $s = (x_{glove}, a_{blue}) \in \mathbb{S}$,
maps to the textual phrase ``blue glove" 
and to an image of a blue glove. 

By convention, any $g^{<loc>}$ declares the location of the preceding object in the clause. That is, $f_{S,T}((x, g_{right})) = \text{``$x$ to the right"}$. Similarly, any $g^{<rel>}$ in between two objects influences both in that order: $f_{S,T}((x, g_{below}, y)) = \text{``$x$ below $y$"}$. Combinations of concepts are shown in green in Fig.~\ref{fig:defn_explain}. 
For simplicity, we denote $\textbf{i}(f_{S,I}((x,a))) = \textbf{i}(x_a)$ unless otherwise specified. 
Similarly, we denote scenes composed of multiple objects as $ \textbf{i}(f_{S,I}((x,y))) = \textbf{i}(x,y), \forall x,y \in \mathbb{V}$ as shorthand.

\end{definition}




\begin{table}[t!]
     \centering
     \label{tab:sum_notaitons}
     \resizebox{0.99\linewidth}{!}{
     \begin{tabular}{c|l}
         \toprule
         Notation & Explanation  \\
         \midrule
        
         $\textbf{i}(\cdot), \textbf{t}(\cdot)$ & CLIP image and text embeddings \\
         $\mathbb{I},\mathbb{T}$ & Set of images and texts\\
         $\mathbb{V}$ & Set of atomic object concepts\\
         $\mathbb{A}$ & Set of atomic attributes\\
         $\mathbb{G}$ & Set of compositional relations\\
         $\mathbb{S}$ & Finite ordered combinations of elts from $\mathbb{V},\mathbb{A}$, $\mathbb{R}$\\
         $\textbf{i}(x_a)$ & Image embedding of object $x$ with attribute $a$ \\
         \bottomrule
     \end{tabular}
    }
     \vspace{-3mm}
    \caption{Summary of notations.}
    \vspace{-6mm}
\end{table}

\subsection{Geometric Requirements for an Ideal CLIP}
We will assume that CLIP encoders can project any image and text as vectors in an ``ideal'' location in the latent space \textit{iff} there exists such a location. By ``ideal location'', we mean the semantic distance between image and text embeddings aligns well with human understanding. 
If under this oracle encoder setting we find inherent contradictions which prohibit an ``ideal" latent space from existing, the conclusions will generalize to the real life setting which is a superset of the concepts we consider, with non-oracle \textbf{i}, \textbf{t}. We justify the need for each condition in the supplements. $C$ is ``ideal'' if it fulfills all following conditions:

\begin{condition}\textbf{(Concept Categorization)}
Satisfaction of this condition requires that
(1.1) $C$ represents basic descriptions and image content. 
\begin{eqnarray*}
\textbf{i}(x)\cdot \textbf{t}(x) &> & \textbf{i}(x)\cdot \textbf{t}(y) \\
\textbf{i}(x,y)\cdot \textbf{t}(x) &> & \textbf{i}(x,y)\cdot \textbf{t}(z) \ \ \ \  \forall \ x \text{, } y \text{, } z\
\in\mathbb{V}  
\end{eqnarray*}
\noindent
(1.2) Images that contain the same semantic concept(s) but differ 
due to an attribute or scene composition, should have higher cosine similarity with each other than with an image that contains a different set of semantic concepts. 
\begin{eqnarray*}
    \textbf{i}(x_a)\cdot \textbf{i}(x_b) &>& \textbf{i}(x_{a})\cdot \textbf{i}(y) \\ \textbf{i}(x,g_1^{<loc>})\cdot \textbf{i}(x,g_2^{<loc>})& >& \textbf{i}(x)\cdot \textbf{i}(y)\\
\forall \text{ } x,y \text{ } \in\mathbb{V} \text{, } \forall \text{ }a,b \in\mathbb{A} \text{, } & & \forall\text{ }g_1^{<loc>}, g_2^{<loc>}\in\mathbb{G}
\end{eqnarray*}
\end{condition}

\begin{condition}\textbf{(Attribute Binding)}
$C$ respects attribute binding. More specifically: (2.1) concepts with different attributes are not parallel in CLIP space. \begin{equation}
     \textbf{i}(x_a)\cdot \textbf{i}(x_b)<1 \quad \forall a,b \in \mathbb{A} \nonumber
\end{equation}

\noindent (2.2) Images representing a concept with a specific attribute are closer in CLIP space to its text embedding.
\begin{equation}
\textbf{i}(x_a)\cdot \textbf{t}(a)>\textbf{i}(x_b)\cdot\textbf{t}(a)
 \nonumber
\end{equation}

\noindent (2.3) Images with the same concepts and attributes present but in different pairings are not parallel in CLIP space.
\begin{equation}
 \textbf{i}(x_a, y_b)\cdot \textbf{i}(x_b, y_a)<1
 \nonumber
\end{equation}



\end{condition}

\begin{condition} \textbf{(Spatial Relationship)}
C respects spatial locations or relationships of objects. This requires that (3.1) images where the same object is in a different location must not have identical embeddings. 
\begin{equation} 
\textbf{i}(x,g_1^{<loc>}) \cdot \textbf{i}(x,g_2^{<loc>}) < 1, \quad \forall g_1^{<loc>}, g_2^{<loc>} \in \mathbb{G}
 \nonumber
\end{equation}


\noindent (3.2) Images with the same objects but in different spatial relationships must not have identical embeddings. 
\begin{equation} 
\textbf{i}(x,g_3^{<rel>},y) \cdot \textbf{i}(x,g_4^{<rel>},y) < 1,\quad \forall g_3^{<rel>}, g_4^{<rel>} \in \mathbb{G} \nonumber
\end{equation}

\noindent (3.3) Images where an object is in the same location or relationship must be semantically closer than images where it is in a different location or relationship.
\begin{equation} 
\textbf{i}(x,g_1,y) \cdot \textbf{i}(x,g_1,z) > \textbf{i}(x,g_1,y) \cdot \textbf{i}(x,g_2,z)
\nonumber
\end{equation}







\end{condition}




\begin{condition} \textbf{(Negation)}
$C$ respects negation. This requires that (4.1) texts and their negated counterparts must have a similarity score lower than any other pairs.
\begin{equation}
    \textbf{t}(x)\cdot\textbf{t}(\neg x)<\textbf{t}(y)\cdot\textbf{t}(\neg x), \quad \forall\ x,y \in \mathbb{T}\nonumber
\end{equation}

\noindent (4.2) An image with some concept must have a lower similarity score with the negated concept text than another text.
\begin{equation}
\textbf{i}(x)\cdot\textbf{t}(\neg x)<\textbf{i}(x)\cdot\textbf{t}(y)\nonumber
\end{equation}

\noindent (4.3) Two distinct negated concepts must have higher cosine similarity than the same distinct non-negated concepts, as they have greater semantic overlap.\footnote{If $|\mathbb{V}| = M$, ``not $x$" and ``not $y$" correctly describe $M-2$ concepts, while ``$x$" and ``$y$" describe none of the same things.}
$$\textbf{t}(x)\cdot \textbf{t}(y)<\textbf{t}(\neg y)\cdot\textbf{t}(\neg x)$$


\end{condition}


\section{Geometric Contradictions in CLIP}
We aim to prove that if some $C$ meets Condition 1, it cannot meet any other condition to be ideal. Therefore no $C$ exists which is ideal. Due to space limitations, we only illustrate in detail how fulfillment of Condition 1 prohibits fulfillment of Condition 2 then summarize the rest. Please refer to the Supplements for detailed proofs of the other conditions.


\subsection{Contradiction for Conditions 1 and 2}


\begin{lemma}
\textbf{Embeddings of images or texts with two obect concepts must be a linear superposition of the respective single object concept embeddings.}
\end{lemma}

\begin{proof}
We derive the ideal placement for $\textbf{i}(x^1, x^2)$ to satisfy Condition 1.1. For this proof, instead of $x,y$ we denote distinct concepts with superscripts: $\textbf{i}(x^1), \textbf{i}(x^2),... $ for $k\in[1,M]$ (to avoid confusion with attributes, which are denoted as subscripts.)
The condition states $\textbf{i}(x^1, x^2)$ must have high cosine similarity with $\textbf{i}(x^1)$ and $\textbf{i}(x^2)$, and low cosine similarity with all other $\textbf{i}(x^j)$.
More formally, $C$ must solve the following optimization problem:
\begin{equation}
\resizebox{.90\linewidth}{!}{$
\begin{gathered}
     \textbf{i}(x^1, x^2) = \operatorname*{argmax}_{\textbf{i}(x^1, x^2)} \Big[\textbf{i}(x^1, x^2) \cdot \textbf{i}(x^1) + \textbf{i}(x^1, x^2) \cdot \textbf{i}(x^2) \\- \sum_{j=3}^{M} \textbf{i}(x^1, x^2) \cdot \textbf{i}(x^j)\Big]\ \  \text{s.t.}\ \  \left\|\textbf{i}(x^1, x^2)\right\| = 1
\end{gathered}
$}
\end{equation}

Here, the first two terms guide the local placement of $\textbf{i}(x^1, x^2)$, while the last term introduces a global constraint to avoid proximity to other embeddings. The constraint ensures that all embeddings must lie on the unit hypersphere.
We can expand the sum to see that:
\begin{equation}
\begin{split}
     \textbf{i}(x^1, x^2) = \operatorname*{argmax}_{\textbf{i}(x^1, x^2)} \Big[ \textbf{i}(x^1, x^2) \\\cdot \Big(\textbf{i}(x^1)+ \textbf{i}(x^2) + \textbf{i}(x^1)+ \textbf{i}(x^2) - \sum_{j=1}^{M} \textbf{i}(x^j)\Big)\Big]
     \end{split}
\end{equation}


\noindent Since random vectors in high dimensions will be approximately symmetrically distributed,
$\sum_{j=1}^{M} \textbf{i}(x^j) \approx 0$.
The optimum is then reached when $\textbf{i}(x^1, x^2)$ is parallel to $\textbf{i}(x^1)+ \textbf{i}(x^2)$. 
Thus we see $\textbf{i}(x^1, x^2)$ is a normalized superposition of $\textbf{i}(x^1)$ and $\textbf{i}(x^2)$, and lies on the geodesic arc between \( \textbf{i}(x^1) \) and \( \textbf{i}(x^2) \) on the hypersphere, i.e.,
\begin{equation}
\boxed{
\begin{gathered}
\textbf{i}(x^1,x^2) = \frac{  \textbf{i}(x^1) + \textbf{i}(x^2)}{\left\|\textbf{i}(x^1) + \textbf{i}(x^2)\right\|}    
\end{gathered}
}
\label{eq:sup}
\end{equation}



\end{proof}

\begin{lemma}
\textbf{$C$ cannot distinguish between different attribute bindings. That is, $\textbf{i}(x_a, y_b) = \textbf{i}(x_b, y_a)$}.
\label{lemma_attbind}
\end{lemma}

\begin{proof}
We start by deriving the ideal location for $\textbf{i}(x_a)$ in $C$ where the distance between this unit vector and $\textbf{i}(x)$ and $\textbf{t}(a)$ is strictly semantically correct per Conditions 1 and 2.1-2.2. Then we derive $\textbf{i}(x_a,y_b)$ and $\textbf{i}(x_b,y_a)$ to meet Condition 1.1, which will contradict Condition 2.3. 

Condition 1.2 states that images of the same object that differ only in attributes (\emph{e.g.,} red car, black car) should be more similar to each other in $C$ than than images of distinct objects. 
This means that the attribute-specific image embedding can be expressed as a small perturbation of the representative image embedding for that object:
\begin{equation}
    \textbf{i}(x_a) = (1-\delta)\textbf{i}(x) + \mathbf{v}
    \label{smallv}
\end{equation}




\noindent where $\delta \ll 1$ is some small positive constant and vector $\mathbf{v}$ denotes the attribute-specific change such that $\left\|\textbf{i}(x_a)\right\| = 1$ and $\textbf{i}(x_a)\cdot \textbf{i}(x) \geq 1-\delta$. 

Now we consider
two objects $x$ and $y$ and two attributes $a$ and $b$ where the objects are agnostic to the attributes. More precisely, we consider $x,y,a,b$ where the semantic distance between each attribute's text embedding and both objects is equal, \emph{i.e., } $\textbf{i}(x)\cdot\textbf{t}(a) = \textbf{i}(y)\cdot\textbf{t}(a) = \cos(\theta)$ and $\textbf{i}(x)\cdot\textbf{t}(b) = \textbf{i}(y)\cdot\textbf{t}(b) = \cos(\omega)$. A concept quartet that satisfies this criterion could be $x$:car, $y$:ball, $a$:red, $b$: black. 

Condition 2.2 states that $\textbf{i}(x_a)\cdot \textbf{t}(a) > \textbf{i}(x)\cdot\textbf{t}(a)$. Therefore, a large component of $\mathbf{v}$ in Eq.~(\ref{smallv}) must be in the direction of $\textbf{t}(a)$.
Below, we will show that if we make the strong assumption that $v = p\textbf{t}(a)$, then $\textbf{i}(x_a, y_b) = \textbf{i}(x_b,y_a)$. Then we show that the result will be robust even if $\mathbf{v} = p\textbf{t}(a) + \boldsymbol{\epsilon}$ for some noise vector $\boldsymbol{\epsilon}$.
Letting $\mathbf{v} = p\textbf{t}(a)$, we write
$\textbf{i}(x_a)$ as a superposition:
\begin{equation}
\boxed{
\textbf{i}(x_a) = (1 - \delta) \textbf{i}(x) + p \textbf{t}(a)
}
\end{equation}

We derive $p$ in Sec E.4 of the Supplements as:
\begin{equation}
\resizebox{.90\linewidth}{!}{$
p = -(1 - \delta) \cos \theta \pm \frac{1}{2}   \sqrt{4(1 - \delta)^2 \cos^2 \theta + 8\delta - 4\delta^2}   
$}
\end{equation}

Notice that $p$ only depends on $\theta$, where $\theta = \arccos(\textbf{i}(x)\cdot \textbf{t}(a)) = \arccos(\textbf{i}(y)\cdot \textbf{t}(a))$.
By the same reasoning as above, for object $y$ with attribute $a$, we have:
\begin{equation}
    \textbf{i}(y_a) = (1 - \delta) \textbf{i}(y) + p \textbf{t}(a)
\end{equation}
Similarly for attribute $b$, $\textbf{i}(x_b)$ and $\textbf{i}(y_b)$ share the same weighting factor $q$ for $\textbf{t}(b)$ as follows:
\begin{equation}
\begin{split}  
\textbf{i}(x_b) = (1 - \delta) \textbf{i}(x) + q \textbf{t}(b)\\
\textbf{i}(y_b) = (1 - \delta) \textbf{i}(y) + q \textbf{t}(b)
\end{split}
\end{equation}
Now consider the composite image embedding $\textbf{i}(x_a, y_b)$. Per Eq.~(\ref{eq:sup}), it is decomposed into a superposition:
\begin{equation}
\boxed{
\begin{aligned}
    \textbf{i}(x_a, y_b) &=\frac{(1 - \delta)( \textbf{i}(x)+\textbf{i}(y)) + p\textbf{t}(a) + q\textbf{t}(b)}{2} \\
    &= \textbf{i}(x_b, y_a)
    \end{aligned}
    }
    \label{attbind_eq}
\end{equation}
This shows that the composite embedding is identical regardless of which object is paired with which attribute. In other words, the distinct attribute bindings become indistinguishable in the final embedding. This equivalence violates Condition 2.3, which requires that different attribute-object bindings produce distinct embeddings.

For cases where $\mathbf{v} = p\textbf{t}(a) + \boldsymbol{\epsilon}$ for some random perturbation vector $\boldsymbol{\epsilon}$, we show in Sec E.3. of the Supplements that the results in Eq.~(\ref{attbind_eq}) are robust to the addition of $\boldsymbol{\epsilon}$.
\end{proof}

\subsection{Contradictions to Other Conditions}

Now we proceed to illustrate how the linear superposition implied by Condition 1 contradicts other conditions. If Condition 1 is fulfilled, then the following impossibilities occur:
\begin{itemize}
    \item \textbf{Condition 3:} Spatial relationships between objects must embed as   
   $\textbf{i}(x, g^{<rel>}, y) = 
   (1-\delta)\,\textbf{i}(x,y) + \textbf{e}_{\perp}$ per Condition 1.2.
   We show that for a simple case of 3 images, where the same localization is present in image 1 and 2 and the same relationship is present in image 1 and 3, we encounter
   $\textbf{e}_{\perp,1} \cdot \textbf{e}_{\perp,2} = - \textbf{e}_{\perp,1} \cdot \textbf{e}_{\perp,3}$. This means Conditions 3.1 and 3.2 cannot be simultaneously satisfied.
    
    \item \textbf{Condition 4:} 
    We find that satisfaction of Conditions 4.1 and 2 in $C$ necessitates 
    $\textbf{t}(\neg x) = -\textbf{t}(x)$. This produces the following effect: $\textbf{t}(\neg x^j)\cdot \textbf{t}(x^k) > \textbf{t}(\neg x^j) \cdot \textbf{t}(\neg x^k)$, violating Condition 4.3.
    \end{itemize}
    
\noindent Please refer to the Supplements for detailed proofs.





\section{Rescuing the CLIP Latent Space}


Is CLIP beyond rescue, or can its learned embeddings be improved? Although its latent space lacks compositional expressivity, its ease of use and powerful image-text organization are undeniable. Hence, we investigate methods to make use of the rich information learned by CLIP for a more principled evaluation of the text to image semantic distance.

Our findings indicate that CLIP is structurally flawed. Re-training, fine-tuning, or simply re-projecting CLIP embeddings will still yield a latent space that lacks the desired properties. Similarly, applying an alternative analytical or learned scoring mechanism on top of the learned embeddings cannot work if two distinct images and texts embed to the same location in $C$ per Lemma \ref{lemma_attbind}. 

\textbf{\textit{Thus: any fix must alter the fact that CLIP represents  images and texts as unit vectors.}} Rather than designing an approach that imposes re-training a full model, we explore an extension to CLIP that is based on the existing CLIP encoders. More specifically: we explore a solution based on three ideas. First, retain the full token and image patch embeddings from CLIP. Second, score matches through a learned mechanism, rather than using cosine similarity. Third, introduce constant, rather than learned, representations of spatial relationship words. 
These ideas are discussed in detail in the next section.

\subsection{Dense Cosine Similarity Maps}
\label{modeltrain}
\begin{figure}[b]
\vspace{-7mm}
    \centering{\includegraphics[width=0.99\linewidth]{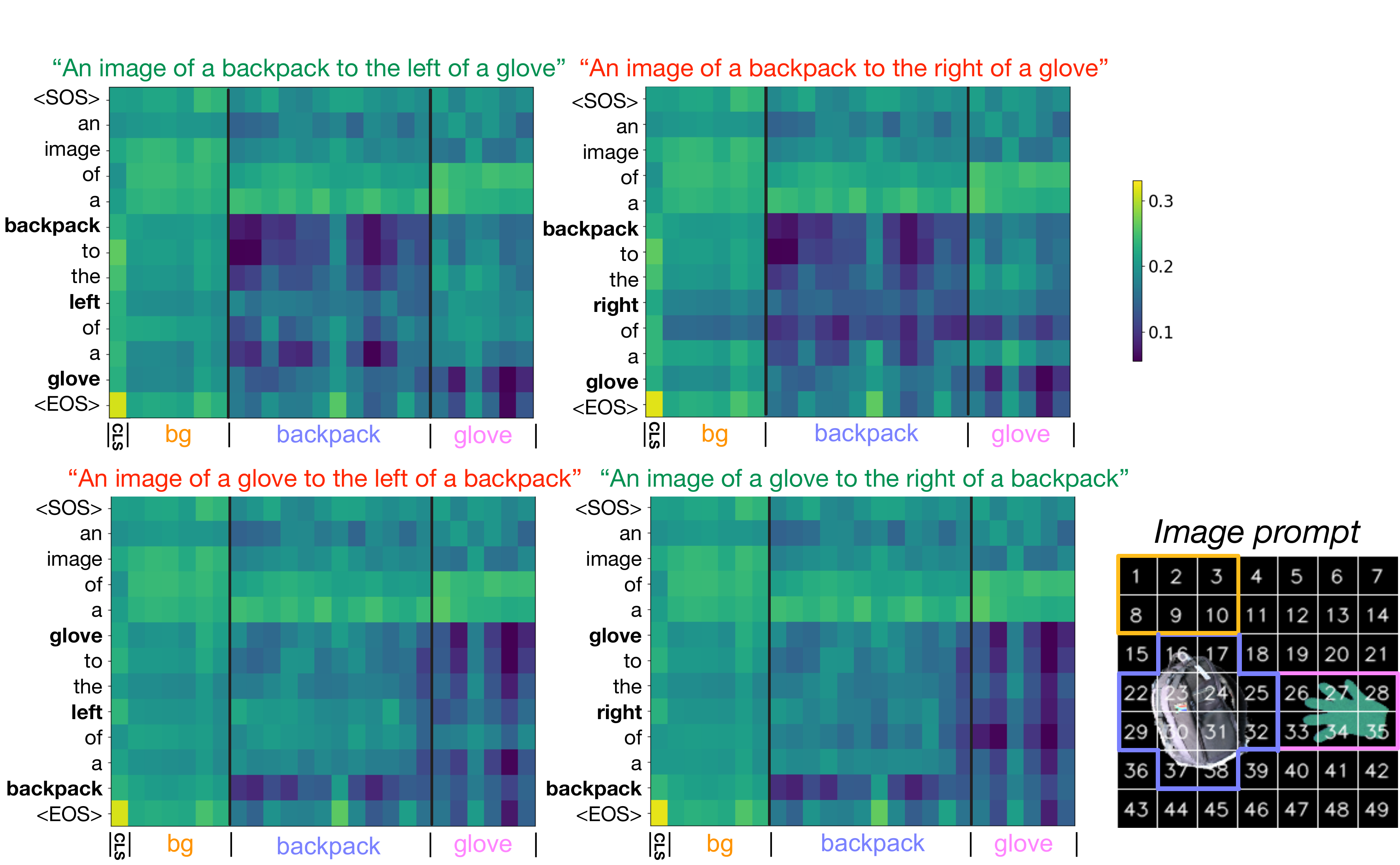}}
    \vspace{-2mm}
    \caption{{\bf Empirical Dense Cosine Similarity Maps.}
    Each one of four matrices shows a DCSM between a different sentence and the same pictured image. For each subfigure: the y axis shows the text tokens and the x axis varies by image patch. We cluster the patches by region as shown in the image.
    Each pixel value is the cosine similarity score between that token and patch embedding. 
    Green sentences correctly represent the image and red ones do not.
    }
    \label{fig:densemap-empirical}
\end{figure}

\begin{figure*}[t]
    \centering
\includegraphics[width=0.95\linewidth]{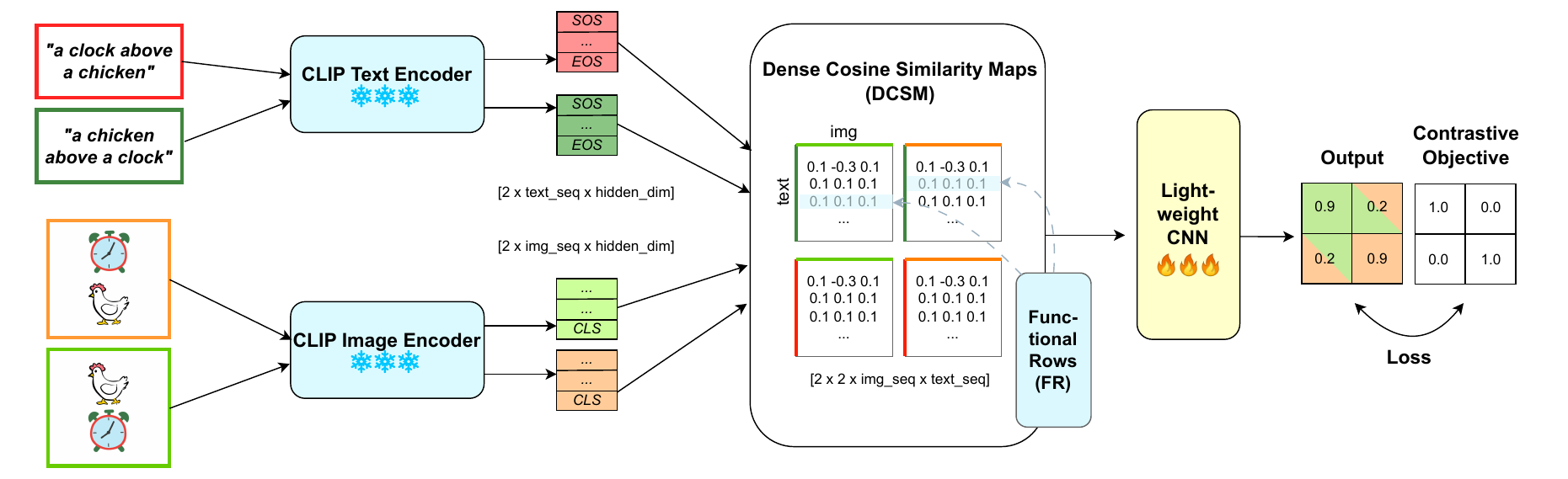}
    \vspace{-2mm}
    \caption{{\bf Schematics of our proposed pipeline and training of its scoring function.} Every sample seen during training contains one hard positive caption and image pair, and a hard negative caption and image pair. Images and texts are passed through frozen CLIP encoders to compute the DCSMs, and then functional rows (FRs) for compositional words are inserted before the DCSMs are scored.}
    \label{fig:full_model}
    \vspace{-4mm}
\end{figure*}

Instead of deriving a single-point CLIP embedding from EOS token from the text embeddings and the CLS token from the image embeddings, we propose to compute the pairwise cosine similarity between all text tokens and all image patches, and then to use a learned scoring mechanism on the resulting dense cosine similarity map (DCSM). The intuition is that transformers~\cite{vaswani2017attention}, in learning the correct CLS and EOS token pair, store useful information in the token and patch level embeddings. While dense image patches have been used in tasks like image segmentation or dense label generation ~\cite{zhou_extract_2022, denseclip_rao_2021}, our approach is the first to extract embeddings at both the token and patch level, and densely compute the cosine similarity across the two representations to formulate the task of image-text pair scoring as a pattern recognition problem.

\begin{figure}[b]
  \vspace{-6mm}
\centering{\includegraphics[width=0.9\linewidth]{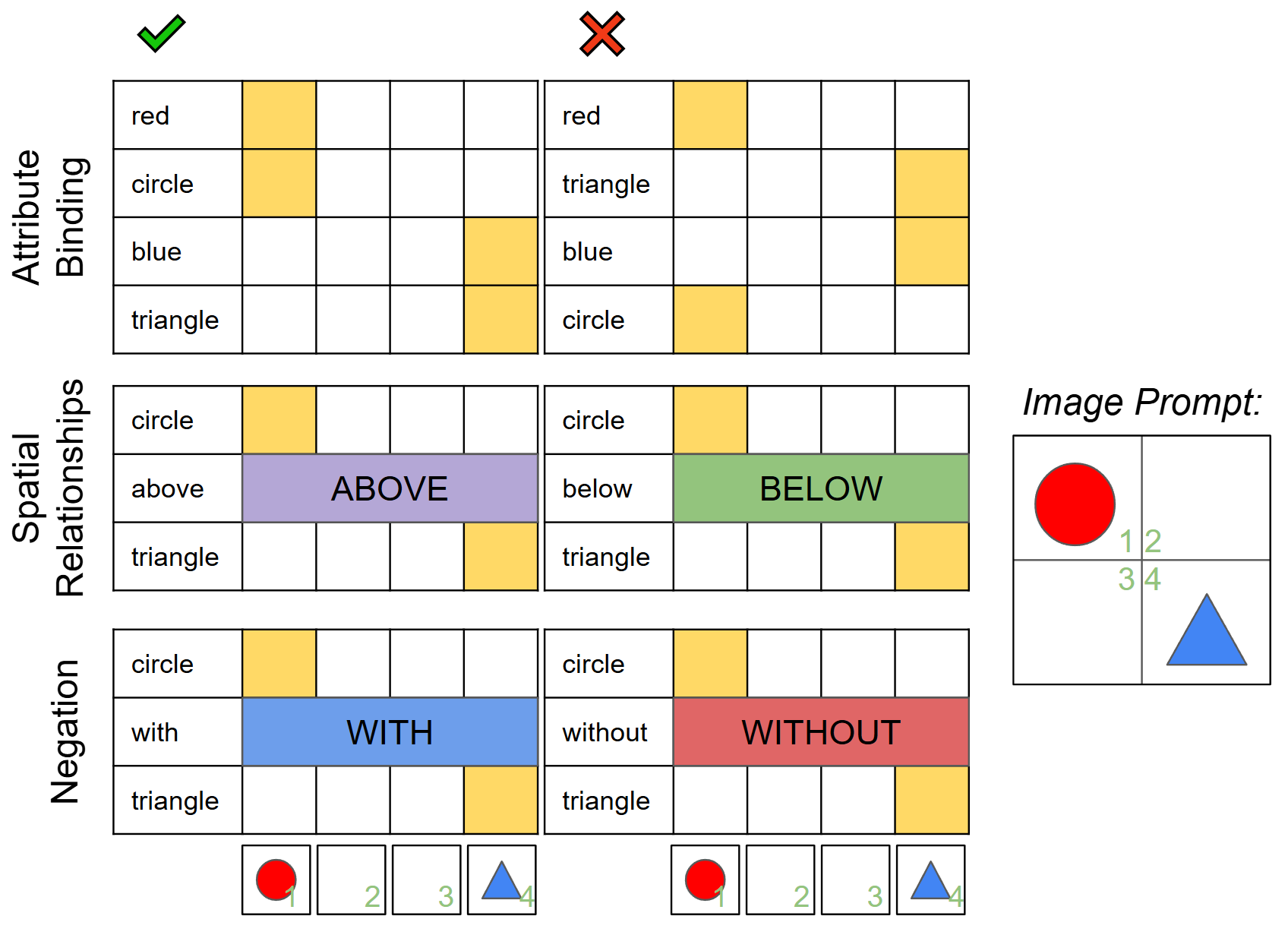}}
    \vspace{-2mm}
    \caption{\textbf{Graphic illustration of retained topology in DCSMs.} The image prompt has 4 patches, and each DCSM shows the dense cosine similarity between each of its patches and the text tokens of the text prompts. Sentences under the green checkmark are a correct pair with the image, while those with a red X are incorrect.
    }
    \label{fig:densemap-graphic}
\end{figure}
Fig.~\ref{fig:densemap-empirical} depicts dense DCSMs for four distinct sentences and one image prompt. The columns have been rearranged to cluster background (with CLS patch being the first column), backpack, and glove patches. DCSMs clearly encode object and attribute locality. As expected from the contrastive training procedure, the highest score in each DCSM is between the EOS token and CLS patch embedding. Interestingly, text tokens describing a specific object have low cosine similarity scores with the patches containing that object. Relatedly, some studies have observed that CLIP attention maps appear inverted, and that post-hoc architectural "surgery" can produce more interpretable saliency maps \cite{li2025closer, bousselham2024grounding}. 
As our goal is to maximally make use of unaltered CLIP features, we do not perform these surgeries and rely on the scoring tool to recognize inverted patterns. Further analysis in this direction will require evaluating the impact of extracted DCSMs from edited CLIP models. 




\noindent\textbf{Functional Rows (FR).} We observe that for both correct and incorrect sentence-image pairs, the DCSMs often look very similar. Much like background patches in an image, text embeddings for compositional concepts, which we will refer to as \textit{functional words} (\emph{e.g.,} ``left of",``right of") that lack immediate visual counterparts only contribute spurious information to the DCSM as the joint encoder training emphasizes words with direct visual references.



This motivates our proposed \textit{functional rows} for DCSMs.
Specifically, we discard the original rows containing functional words in DCSMs, and replace them with a constant vector with randomly chosen entries.




\begin{table*}[t]
\centering
 \resizebox{0.99\linewidth}{!}{
\begin{tabular}{l|ccc|ccc|cc}
        \toprule
        \multicolumn{1}{c}{} & \multicolumn{3}{c}{Attribute Binding} & \multicolumn{3}{c}{Spatial Reasoning} & \multicolumn{2}{c}{Negation} \\
        \hline
               \textbf{Model Name} &  \textbf{CLEVR$_{bind}$} & \textbf{NCD} &  \textbf{VG\_attr} &  \textbf{WhatsUp} &  \textbf{COCO$_{1\&2obj}$} &  \textbf{VG$_{1\&2obj}$} &  \textbf{NegBench$_{coco}$} &  \textbf{NB$_{voc}$}\\
\midrule
    CLIP$_{ViTB/32}$ &        22.2 &            71.3 &     61.3 &     31.9 &         47.0 &       47.1 &      39.2 &     38.3 \\
    CLIP$_{ViTB/16}$ &        20.2 &            63.1 &     61.8 &     30.5 &         48.9 &       51.5 &      41.5 &     37.9 \\
                  NegCLIP &        21.8 &            79.2 &     \textbf{72.4} &     33.2 &         46.1 &       47.2 &      31.4 &     26.5 \\
                     CoCa &        19.2 &            48.7 &     50.8 &     24.5 &         48.6 &       49.5 &      21.6 &     20.8 \\
                     BLIP &        11.1 &            56.2 &     59.4 &     24.4 &         48.5 &       50.4 &      20.5 &     20.7 \\
                   SigLIP &        13.3 &            53.6 &     48.4 &     26.0 &         47.4 &       51.1 &      26.3 &     29.7 \\
\rowcolor{blue!5}\textbf{DCSM$_{synth}$} &        31.0 &            93.6 &     60.9 &     62.6 &         65.6 &       64.4 &      38.8 &     35.2 \\
\rowcolor{blue!5}\textbf{DCSM$_{coco}$} &        \textbf{39.9} &           \textbf{95.7} &     68.1 &   \textbf{ 63.7} &         \textbf{72.4} &      \textbf{67.0} &     \textbf{48.6} &     \textbf{49.0} \\ \midrule
            Random Chance &        25.0 &            50.0 &     50.0 &     25.0 &         50.0 &       50.0 &      25.0 &     25.0 \\
\bottomrule
\end{tabular}
}
    \vspace{-3mm}
    \caption{\textbf{Accuracy comparison across different models on various benchmarks}. The top row categorizes each benchmark into the condition being addressed. Our DCSM takes CLIP ViTB/16 as the base model. Best scores for each dataset are bolded. Evaluation datasets contain scene compositions and attribute, spatial, and negation words not included in the training set for out model.
    } 
    \label{tab:model_comparison}
    \vspace{-4mm}
\end{table*}

\noindent\textbf{Model Pipeline} (see Fig.~\ref{fig:full_model})\textbf{.} Unlike classical CLIP scoring which discards all but the EOS and CLS tokens to produce a single scalar, our proposed method retains all token embeddings and projects them into a joint latent space to compute the DCSM. For all experiments, we use a lightweight CNN with 2 convolutional layers and a hidden dimension of 128, yielding a 20-fold reduction in parameters. 

\noindent\textbf{Training.} The network is trained with a batch size of 8, which is 4000 times smaller than CLIP's original 32,768. Our total training data per model is around 20,000 samples,  which is actually 1.5 times smaller than the mini-batch size of CLIP. 
For a small set of functional words (elements of $f_{G,T}(g)$ for all $g \in \mathbb{G}$), we keep a dictionary of unique fixed FRs and overwrite the corresponding rows in the DCSM. As illustrated in Fig.~\ref{fig:densemap-graphic}, in an idealized setting, DCSMs with FRs allow for clear disambiguation between correct and incorrect image-text pairs, reducing the task of parsing visuolinguistic semantics to  pattern recognition. Notably, our model does \textit{not} see any input image or text at all; it simply learns to recognize the syntactic patterns from the DCSMs.

We train one model on synthetic data created from Objaverse \cite{deitke_objaverse_2022} renders and another on a subset of COCO2017. Additional training details are provided in the Supplements.





\subsection{Performance Evaluation}

We compare our method against a range of models which utilize a joint vision-language embedding space.
Namely: CLIP from OpenAI \cite{radford_learning_2021} and OpenCLIP \cite{cherti2023reproducible}, NegCLIP \cite{yuksekgonul_when_2023} which is a finetuned version of OpenCLIP with mined hard negatives, Coca \cite{yu2022coca} which is an image-text encoder-decoder model trained with contrastive and captioning loss, SigLIP \cite{zhai_sigmoid_2023} which uses a sigmoid loss in lieu of cross entropy, and BLIP \cite{BLIP_li} which bootstraps noisy web scraped data with synthetic, filtered captions. 
The models are evaluated on datasets of three categories: CLEVR-bind \cite{lewis_does_2024}, a composite form of Natural Colors Dataset (NCD) \cite{anwar_image_2024} per \cite{yamada_when_2024}, and VG\_attribution from Attribute, Relation, Order (ARO) \cite{yuksekgonul_when_2023} for attribute binding; WhatsUp, COCO-QA, and VG-QA \cite{kamath_whats_2023} for spatial relationships and localization; and NegBench \cite{alhamoud_vision-language_2025} for negation. Table \ref{tab:model_comparison} shows the results.

Our method significantly outperforms CLIP-like baselines, likely due to the increased dimensionality of DCSMs, which replace scalar cosine similarity with a dense topological map. Specifically, patch indices on the image axis store spatial information, while the token indices on the text axis preserve the semantic ordering. DCSMs trained on COCO outperform the baselines across all datasets except on VG-attribution \cite{yuksekgonul_when_2023}, where NegCLIP benefits from similar hard negatives. Improvements on NegBench are lower, likely due to the data gap between our templated training captions and the natural language captions. 

Notably, our model generalizes well to unseen attribute, spatial, and negation concepts despite the limited training set. In particular, our model is only trained on templated two-object captions but works well for single object captions on CLEVR-bind and subsets of COCO and VG-spatial. This suggests the downstream network learns syntactic patterns rather than overfitting to the training templates. Beyond performance gains, DCSMs offer greater interpretability. Unlike naive CLIP embeddings whose actual values can be arbitrary, DCSMs are human interpretable (see Fig.~\ref{fig:densemap-empirical}), which makes downstream usage more intuitive.

For detailed ablation studies on the impact of dense maps, FRs, classification performance, and scaling analysis, please refer to Sec. C of the Supplements.


\subsection{Generalization to Natural Language}

To test the generalizability of DCSMs to open vocabulary settings, we incorporate LLMs-in-the-loop to dynamically update the lookup table for FRs and reformat natural language sentences for more compact DCSMs. 
We explore this in a toy setting by training our lightweight CNN from scratch with either 5k or 10k COCO images whose captions have had their nouns swapped. 
We choose this particular type of intervention as CLIP-like VLM performance across the board was lowest for this category of hard negatives in Sugarcrepe \cite{hsieh_sugarcrepe_2023}.
On top of the training pipeline described in Sec.~\ref{modeltrain}, we prompt gpt-4o-mini to extract newly encountered functional words in the training data. We evaluate this paradigm on the swap-object split of Sugarcrepe and the VG-spatial split of VL-Checklist~\cite{zhao_vl-checklist_2023}, against naive and finetuned CLIP. Details for training and prompting are in the Supplements.
Results are shown in Table \ref{tab:llm_evals}.

\begin{table}
    \centering
    \begin{tabular}{l|c c}
        \toprule
         Model & VLC$_{vg}$ $_{spatial}$ & SC$_{swap}$ $_{obj}$ \\
         \hline 
    CLIP$_{ViTB/16}$& 50.8& 60.2 \\
    CLIP$_{ft. synth}$& 56.9 & 50.0 \\
    \rowcolor{blue!5}\textbf{DCSM}$_{synth}$& 50.6&  59.8 \\
    \rowcolor{blue!5}\textbf{DCSM}$_{open}$ $_{vocab}$& \textbf{63.5} & \textbf{63.8} \\
         \bottomrule
    \end{tabular}
    \vspace{-3mm}
    \caption{\textbf{Open Vocab Accuracy}. Sugarcrepe results with LLM prompt simplification are averaged across 6 trials. VL-Checklist captions were not simplified at test time. We report the higher score of the two trained models 
    for open-vocab DCSM. 
    }
    \label{tab:llm_evals}
    \vspace{-6mm}
\end{table}

VL-checklist captions are short but contain novel functional words not present in the static lookup table. 
Sugarcrepe captions are longer, with complex syntactic arrangements as well as novel functional words. 
We see that the finetuned CLIP model overfits to simple clauses and fails at the more complex setting, while the closed vocabulary DCSM performance is consistent with that of naive CLIP. The open vocabulary DCSM shows a modest performance increase in both settings.
This suggests that LLM for dynamic FR update is a promising avenue, though long winded prompts are still a bottleneck for the lightweight CNN model. 
We predict that increasing the dataset size and quality, as well as improving scoring model complexity will help refine natural language performance.

\section{Conclusions}
Through formal logical analyses, we show that inherent geometric limitations in CLIP prevent the correct representation of image content, attribute binding, spatial localization, and negation.
To address this problem we tap the information from CLIP's text and image encoders, creating a richer latent space with Dense Cosine Similarity Maps and Functional Rows. We evaluate our simple solution against a wide array of benchmarks and find that it surpasses SoTA performance. Future work includes scaling up our prototyped pipeline to expand on its open vocabulary potential with more sophisticated CLIP base models, scoring models, LLMs, and training data. Further analyses to discover other fundamental properties of latent geometries and to explore alternative embedding manifolds beyond unit spheres will inform the design of next‐generation VLM architectures.

\section*{Acknowledgments}
We would like to thank Yisong Yue, Damiano Marsili, Laure Delisle, Neehar Kondapaneni, Sevan Brodjian, Anna Ding, and Meir Yossef Levi, for valuable discussions and proofreading. Raphi Kang is supported by the NSF Graduate Research Fellowship. Yue Song is supported in part by gifts from Cisco and OpenAI. Georgia Gkioxari is supported by the Hurt Scholar program, Meta and Google.

{
    \small
    \bibliographystyle{ieeenat_fullname}
    \bibliography{main}
}

\newpage
\appendix
\startcontents[appendices]
\printcontents[appendices]{l}{1}{\section*{Is CLIP ideal? No. Can we fix it? Yes!\\--Supplementary Material--}}

\section{Extended Related Work}

\noindent\textbf{Vision Language Models.} Recent advancements in VLMs have significantly bridged the gap between vision and language. A seminal work in this area is CLIP~\cite{radford_learning_2021}, which demonstrated that large-scale contrastive learning can effectively align image and text representations into a shared embedding space. 
Building upon its core principles of large-scale contrastive pretraining and joint representation learning, several subsequent works have explored alternative architectures and training paradigms, including ALIGN~\cite{jia2021scaling}, FILIP~\cite{yaofilip2022}, SLIP~\cite{mu2022slip}, ALBEF~\cite{li2021align}, and CoCa~\cite{yu2022coca}, to name a few. 
Autoregressive VLMs ~\cite{liu2024improved,liu2024visual} have emerged as compelling alternatives to CLIP by jointly attending to both text and image embeddings.
Neurosymbolic program synthesis methods like ViperGPT \cite{surís2023vipergptvisualinferencepython}, VisProg \cite{gupta2022visualprogrammingcompositionalvisual}, and VADAR \cite{marsili2025visualagenticaispatial} also mitigate some of CLIP's limitations by formulating the text-image semantic distance acquisition as several subproblems. While these methods are more comprehensive than CLIP and can specialize in complex visual reasoning, they are orders of magnitude more expensive to infer and do not offer the same simplicity or gradient retention as CLIP, limiting their downstream applicability. In fact, most of these models rely on a CLIP-like model's latent space as a submodule. As such there remains a strong motivation to continue refining and extending CLIP-like architectures by addressing their inherent shortcomings.

\noindent\textbf{Empirical Limitations of CLIP.} A growing body of work has revealed several limitations of CLIP in handling complex visual-text interactions. One major issue is its difficulty in distinguishing between different attribute-concept bindings in multi-object scenes~\cite{lewis_does_2024, newman_pre-trained_2024, campbell_understanding_2024}. For example, the text prompt ``A purple sphere" will have very high cosine similarity with an image that contains a purple cube and a yellow sphere, despite the yellow attribute not belonging to the spherical object. Lewis~\emph{et al.} propose CLEVR-bind as a simple benchmark which isolates the attribute binding capability of VLMs. More comprehensive natural-language benchmarks targeting attribute binding include Attribute, Relation, and Order (ARO), Sugarcrepe~\cite{hsieh_sugarcrepe_2023}, VL-checklist~\cite{zhao_vl-checklist_2023}, and Multimodal Visual Patterns (MMVP)~\cite{tong2024eyes}. Additional studies have noted that CLIP’s text embeddings often behave like ``bag-of-words" in practice, leading to imsinterpretations of object layouts or conflates multiple entities within a single scene~\cite{yuksekgonul_when_2023, lewis_does_2024,newman_pre-trained_2024}. 
Yuksekgonul~\emph{et. al} specifically propose WhatsUP as a benchmark which isolates spatial reasoning capacities of VLMs. In addition, the suite of compositional understanding benchmarks (ARO, VL-checklist, Sugarcrepe, MMVP) also include captions that require spatial reasoning. Another notable failure mode is CLIP's inability to accurately represent negation~\cite{alhamoud_vision-language_2025,singh_learn_2024}. In response, Alhamoud~\emph{et al.} develop NegBench to specifically assess how well VLMs handle various forms of negatory sentences. The aforementioned compositional benchmarks further challenge models with captions that require proper negation understanding.

Some other criticisms of CLIP are its inability to generalize to different reference frames \cite{zhang_vision-language_2024} or to count \cite{paiss_teaching_2023}. While these issues represent additional challenges for CLIP, they are beyond the scope of our current work. 

\noindent\textbf{Proposed Solutions to CLIP Limitations.} The most prominent method of corrections have been to change the training data distribution, such as retraining or fine-tuning CLIP with hard negative or positive examples~\cite{yuksekgonul_when_2023,wang_boosting_2024,singh_learn_2024,alhamoud_vision-language_2025,paiss_teaching_2023,radenovic_filtering_2023,koohpayegani_genie_2024} or increasing the training data for more comprehensive and longer captions~\cite{wang_advancing_2024,li_laion-sg_2024}.
However, simply scaling the training data to mitigate specific problems will often lead to reduced generalization~\cite{zang_overcoming_2024,ma_understanding_2023,hsieh_sugarcrepe_2023}. 

Some engineering solutions include using object detectors to segment images into smaller ROIs~\cite{li_covlm_2023}, adding attribution tracing for correct text-image pairs during training~\cite{li_interpretable_2024}, patch clustering for semantic segmentation~\cite{shao_explore_2024}, using chain of thought spatial reasoning~\cite{chen_spatialvlm_2024} or altering the self attention mechanism in the vision encoder~\cite{wang_sclip_2023}. Simpler solutions may be querying CLIP with multiple descriptions \cite{menon_visual_2022}. All of these methods require retraining CLIP from scratch, or typically adding a heavy-handed component to single out the objects in a scene. 

\section{Experimental Details}

\subsection{Implementation Details}

By convention the rows of the dense map correspond to one text token embedding, and the columns to one image patch. Every text and image pair creates a DCSM of shape (30,197), where 30 is the maximum number of text tokens and 197 is the number of 16x16 image patches in an image of shape 224x224. Text prompts shorter than 30 tokens are padded with EOS tokens. FRs are therefore of the shape (num-image-patches, 1).
For example, for the sentence ``An image of a circle above a triangle", the word \textit{above} is a functional word and the corresponding row in the DCSM gets replaced with the respective FR in the lookup table. For synonymous functional words, we use a single FR. (Somewhat unusually, we consider ``front","below", and ``behind","above" to be synonymous, as DCSMs use a 2D fixed frame of reference due to the topology being represented by the patch index.) The DCSMs are z-score normalized for stable training.

For all training and experiments, we use a lightweight CNN with 2 convolutional layers and a hidden dimension of 128.
In addition to a 20-fold reduction in parameters, we train our network with a batch size of 8, a 4000-fold decrease from the original mini-batch size of 32,768. In fact, the sum of all our training data per model is smaller than this number. 

Our model outputs a single score for each image and text pair DCSM. During training, we use a contrastive cross-entropy loss as with the original CLIP \cite{radford_learning_2021-1}. We train with the Adam optimizer with learning rate initialized at $1e-3$. One model is trained with a curated synthetic dataset and another with COCO 2017 training split. Dataset curation is detailed below.

\subsection{Datasets}
We train our pipeline on two different datasets - one synthetic dataset composed of open-source 3D assets from Objaverse \cite{deitke_objaverse_2022} placed upon randomized backgrounds, and another generated from COCO-train-2017 \cite{coco_lin_2014}.

We are mainly interested in labeling images with text prompts that lie in the Condition 2,3,4 category. That is, for both the synthetic case and the COCO-train case, we generate a dataset for attribute binding, spatial relationships/localization, and negation. Samples from each dataset are shown in Fig.~\ref{fig:dataset}.

\noindent \textbf{Attribute Binding Dataset} \quad 
Every image in this dataset includes two distinct objects of unique colors or sizes. For each image, we generate a ``hard negative". So if there is an image with a ``red cow and purple ghost", we also generate an image with ``red ghost and purple cow". This means that every sample has a positive and negative image, and two positive and negative captions each. The positive image contains object $A$ with attribute $A_{att}$, and object $B$ with attribute $B_{att}$. The negative image contains the same objects but with swapped attributes. The positive caption options are: (P1) “$A_{att}$ $A$ and $B_{att}$ $B$” and (P2) “$B_{att}$ $B$ and $A_{att}$ $A$”. The negative caption options are: (N1) “$A_{att}$ $B$ and $B_{att}$ $A$” and (N2) “$B_{att}$ $A$ and $A_{att}$ $B$”.
 We generate 5,402 images for this dataset, which makes 2701 samples with associated opposites.

For the COCO-train set, we use a natural language processing library to extract adjective-noun pairs in the natural language captions, and select images that have at least two distinct objects $A,B$ with distinct attributes $A_{att}, B_{att}$. Captions follow the same format as above. We select 8,547 images from COCO-train towards this dataset.

\noindent \textbf{Spatial Relationships and Localization} \quad Similarly as above, we generate synthetic images where one object is placed either above, below, to the left, or to the right of, another object with random jitter. For every positive image where $A$ is $rel$ to $B$, there is a negative image where $A$ is $rel_{opp}$ to $A$. Here, (above, below) are opposite relation pairs, as are (left of, right of). 
The positive caption options are: (P1) “$A$ $rel$ $B$” and (P2) “$B$ $rel_{opp}$ $A$”. The negative caption options are: (N1) “$A$ $rel_{opp}$ $B$” and (N2) “$B$ $rel$ $A$”. We generate 11,324 images for this dataset, which makes 5662 samples with associated opposites.

For the COCO images, we choose images where at least two distinct objects are present, and use the relationship between their bounding boxes to validate that they satisfy the definition of one of the spatial relationships being considered. Note that, as we are using patch location on the image to preserve topology, we use a fixed frame of reference to determine the meaning of ``above", ``below", ``left", and ``right". To make a negative version of the image, we use the CutMix technique \cite{cutmix-yun} to swap the image content in the two bounding boxes.
With the positive images from COCO-train and generated hard negatives, there are 11,502 images in this dataset, which makes 5751 samples with associated opposites.

\begin{figure}
    \centering
    \includegraphics[width=0.99\linewidth]{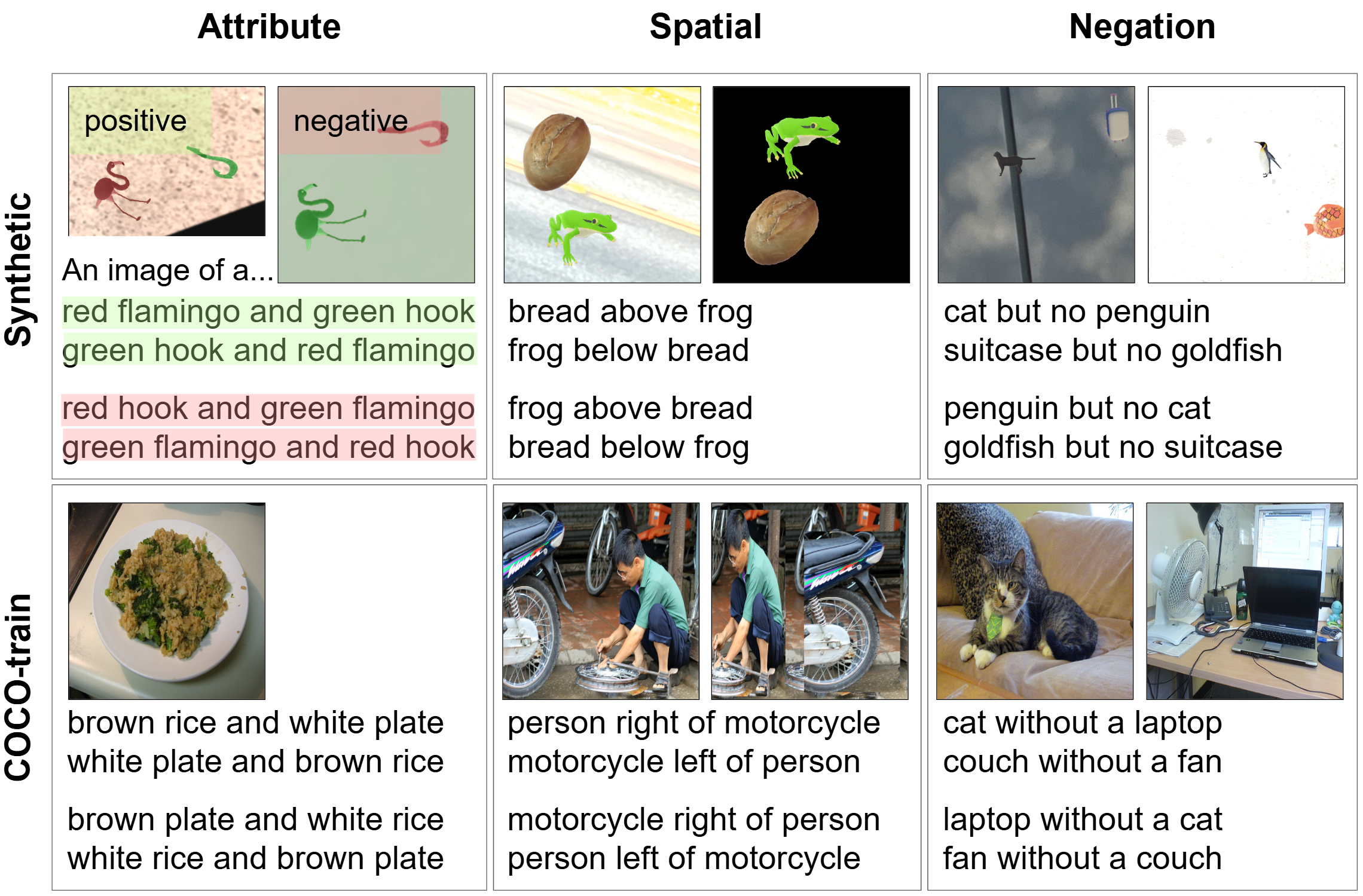}
    \caption{Overview of Dataset Curation. Each box is a dataset. The left image is a positive pair with the top two captions, and the right image is a positive pair with the bottom two captions. The COCO-train Attribute set does not have a hard negative image counterpart.}
    \label{fig:dataset}
\end{figure}

\noindent \textbf{Negaton} \quad Generating negatory captions is tricky.
With contrastive training, an entirely negatory caption (e.g., ``An image without a turtle") necessitates that all other images in the batch be a positive sample (e.g., they must now all contain a turtle). 

As such, for the first pass for the models trained with synthetic data, we generate images with two random objects $A1$ and $A2$, and choose as its hard negative another image which contains two non-overlapping objects, $B1$ and $B2$. For the negatory term $\neg$, we choose between ['but not', 'and no', 'without'].
The positive caption options are: (P1) “$A1 \neg B1$” and (P2) “$A2 \neg B2$”. The negative caption options are: (N1) “$B1 \neg A1$” and (N2) “$B2 \neg A2$”.
We generate 10,000 images for this dataset.

For COCO-train we select two images that have two distinct object labels, and generate captions the same way. We select 10,000 images toward this dataset.


\subsection{Pseudocode}

Below we illustrate the process for extracting DCSMs.






\begin{figure}[htbp]
    \centering
\begin{lstlisting}[language=Python]
dense_image_features = image_encoder(I).last_hidden_state.unsqueeze(1)
# Shape: (batch_size, 1, iseq, embed_dim)

dense_text_features = text_encoder(T).last_hidden_state.unsqueeze(0)
# Shape: (1, batch_size, tseq, embed_dim)

dcsm = einsum( "bqie, lpte -> bpit", dense_image_features, dense_text_features )
# Shape: (batch_size, batch_size, iseq, tseq)

dcsm = add_functional_rows(dcsm, lookup_list)

out = lightweight_cnn(dcsm)
# Shape: batch_size, batch_size

labels = eye(out.shape[0])
# Shape: batch_size, batch_size

loss = CE(out, labels) + CE(out.t(), labels.t())
\end{lstlisting}
    \caption{Pytorch-like code for training our model with DCSMs}
    \label{fig:code}
\end{figure}

\section{Ablation Studies}

\subsection{Impact of DCSMs}

\begin{table}[h]
    \centering
     \resizebox{0.99\linewidth}{!}{
    \begin{tabular}{l|c c}
    \toprule
\textbf{Model} &\textbf{WhatsUP} &	\textbf{COCO$_{2 obj}$}\\
\hline
\rowcolor{blue!5}\textbf{Ours - DCSM (CNN})&	\textbf{62.6}&	\textbf{70.9}\\
\rowcolor{blue!5}\textbf{Ours - DCSM (CNN)$_{ w/o 
 FR}$}	& 48.6 &	55.5\\
\rowcolor{blue!5}\textbf{Ours - DCSM (ViT)}	&29.4&	47.0\\
CLIP - ViTB/16&	30.5	&45.9\\
CLIP - ViTB/16$_{\textit{f.t. synth}}$	&25.5&	49.1\\
CLIP - ViTB/16$_{\textit{MLP scorer}}$	& 25.4 &	53.6\\
\bottomrule
    \end{tabular}
    }
    \caption{The DCSM networks were trained with synthetic data.}
    \label{tab:dcsm_synthetic}
\end{table}

In Sec 5. we said that it follows from our analysis that neither a fine-tuned/reprojected CLIP embedding space, nor a learned scoring module, could alone be the fix to CLIP's fundamental shortcomings. To verify this conclusion, we perform a series of ablations. First, we finetune OpenAI's CLIP (ViTB/16) with the same synthetic dataset used for our model training. We also train a MLP scoring module which takes concatenated text and image embeddings as input to output a score, again on the same synthetic dataset. Both attempts fail to improve performance on WhatsUP and COCO-spatial, resulting in accuracies near chance (25\% and 50\%, respectively). 

Further, we alter our training pipeline in two different ways to assess the need for a CNN as well as the FRs. Removal of the FRs decreases performance overall, but the increased information capacity from using the DCSMs and the downstream network still allows the model to perform above fine-tuned CLIP models. 
Replacement of the CNN with a comparably small ViT, with a patch size of 2 and 2 layers of 4 attention heads, resulted in another near-chance performance on the datasets. The ViT appears more prone to overfitting to the training set, as it does not have the imposed constraint of pattern-recognizing kernels as in CNNs.

All networks were trained with learning rate 1e-3 with Adam for 27 epochs. Under minimal compute, the CNN generalized much better. 
Fine-tuning CLIP projection layers with the same dataset for the same number of epochs did not result in any noticeable performance increase.

\subsection{Classification Performance}

A known problem with VLMs using CLIP embeddings is the decline in classification capacity \cite{zhang_why_2024}. We evaluate our model on two classification datasets and find that, despite being trained with simple prompts and no image classification captions, the reduction in classification performance is minimal compared to those observed in BLIP or LLaVA.

\begin{table}[h]
    \centering
 \begin{tabular}{ l|c c }
 \toprule
\textbf{Model} & \textbf{Caltech101} & \textbf{Flowers102}\\
 \hline
 CLIP-ViT B/16 \cite{radford_learning_2021}  & \textbf{82.6}    & \textbf{67.7} \\
 \rowcolor{blue!5}\textbf{Ours - DCSM + synth}   & \textbf{77.8}  & \textbf{52.9} \\
 \rowcolor{blue!5}\textbf{Ours - DCSM + coco }  & \textbf{79.2}   & \textbf{43.7} \\
 BLIP2-2.7B  \cite{BLIP_li}  & 22.3   & 14.2 \\
 IBLIP-7B \cite{dai2023instructblipgeneralpurposevisionlanguagemodels} \ & 58.4    & 26.8 \\
 LLaVA1.5-7B \cite{liu2024visual}  & 62.1    & 10.2 \\
 \bottomrule
\end{tabular}
    \caption{Classification Accuracy of VLMs. Top three scores per dataset are bolded.}
    \label{tab:classification_vlm}
\end{table}
\noindent

\subsection{Scaling Analysis}
In this work, we showcase a very small and computationally light network and training pipeline for the DCSM method. To verify that this method will scale with increasing data, we perform a scaling analysis. Fig.~\ref{fig:scale} shows the results. The x axis is the approximate number of samples from the curated COCO2017 training set. From this we see that our training pipeline is likely to scale with increasing the dataset.

\begin{figure}[h]
    \centering
    \includegraphics[width=0.85\linewidth]{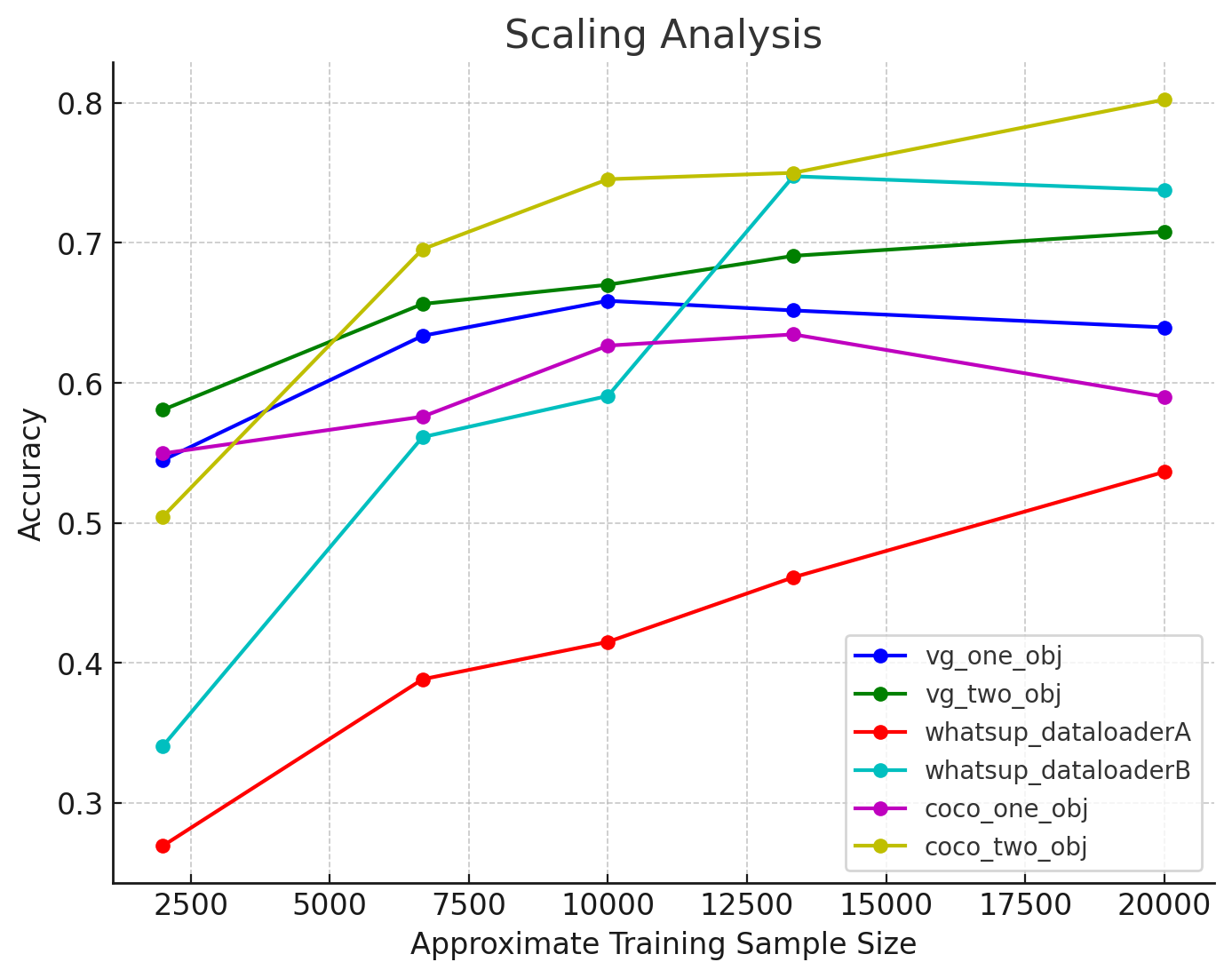}
    \caption{Result of linearly scaling training data. x values are approximate dataset sizes.}
    \label{fig:scale}
\end{figure}

\section{Empirical Observations of CLIP shortcomings}

In Fig.~\ref{fig:superposition_issue} we show the empirical effects of the superposition derived in Lemma 1. In summary, the figure serves to show how an image with an increasingly greater number of objects present embeds increasingly farther from the text label for any one of those objects in the CLIP latent space. The degree of this effect is such that beyond 6-8 objects in one image, CLIP embeddings of random noise images are similarly close to those object text labels.

\begin{figure}[h]
    \centering
{\includegraphics[width=0.49\linewidth]{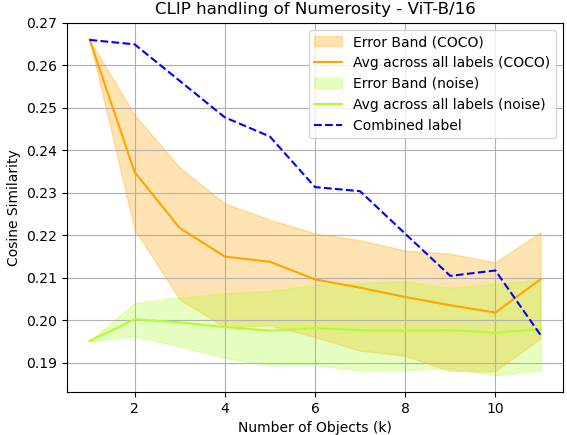}}
{\includegraphics[width=0.49\linewidth]{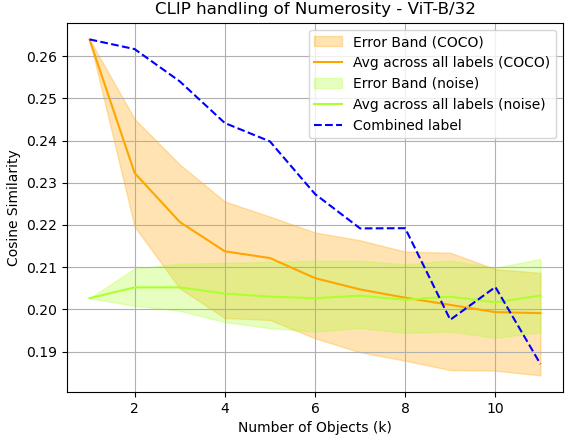}}
{\includegraphics[width=0.49\linewidth]{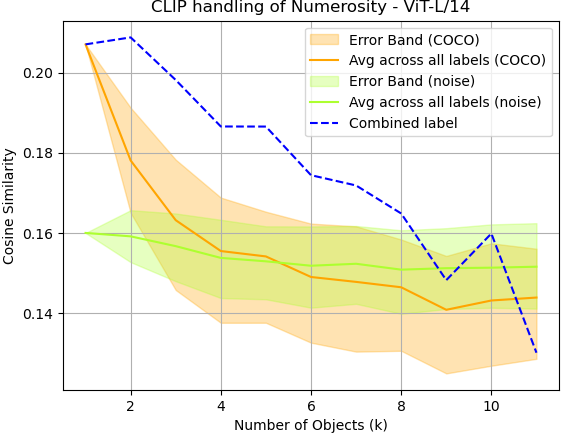}}
{\includegraphics[width=0.49\linewidth]{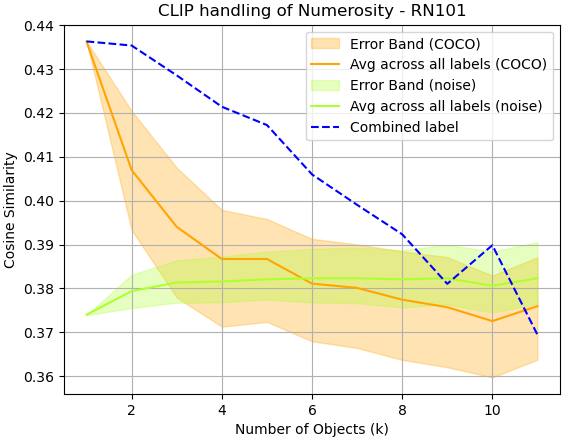}}
    \caption{For each image in COCO-validation, we identify k objects with labels. Then we take the cosine similarity between all k labels and the image. The orange line shows the average cosine similarity for images with k objects and all labels that appear in the image.
    The blue dotted line shows the cosine similarity for a label which combines all k labels and the image being considered, averaged for each k.
    The green plot shows the average cosine similarity score between 5 random noise images and the labels of all COCO-validation images. The error bands indicate the 25th and 75th quartile.}
    \label{fig:superposition_issue}
\end{figure}

\section{Definitions and Proofs}

\subsection{Definitions Addendum}

\textbf{Definition 1.}
For x = the concept of a bird, its mapping in $\mathbb{I}$ is an image of a bird, and its mapping in $\mathbb{T}$ is ``bird". 
For all concepts x in $\mathbb{V}$, there is an equivalence class in $\mathbb{I}$ and $\mathbb{T}$ that unambiguously represent that concept, respectively. 
We narrow the scope of all images to 1 representative image per concept. So for the object concept of a ``bird", there is one image which is equivalent. Similarly, for each concept x there is only 1 corresponding $d \in \mathbb{T}$.

We also choose $M$ object concept elements toward the subset $\mathbb{V}$.
In the manner of large ontologies (such as ImageNet1000), we take it to be true that there exist at least $M\leq$1000 object classes which are mutually exclusive. This means that for any representative image in the set of $M$, an expert human reviewer would be able to assign it to that unique object class. These object classes must be \textit{subordinate categories} on any hierarchical directed tree of visual concepts. (An example concept hierarchy: \textit{car} and \textit{bus} are lower-level concepts that are below \textit{vehicle}.) This ensures no two concepts in $\mathbb{V}$ overlap in semantics.

\subsection{Conditions Addendum}

\textbf{Condition 1.}
Semantic separability means: if human annotators would agree that a text caption appropriately describes an image, the embeddings for that caption and image should have high cosine similarity. Conversely, the score should be low if an annotator deems the caption to be inappropriately matched to the image.

Notably, this condition does not require correct hierarchical organization among concepts or semantically accurate placement of synonymous phrases. In reality there exist more than N distinct concept categories and each category may have synonymous text representations and multiple instances and viewpoints of the representative visual object. Our goal is to set up extremely minimal geometric requirements for $C$ to evaluate whether they are possible to attain.

This condition must be satisfied for $C$ to perform zero-shot image classification, retrieval, and semantic similarity search. Example benchmarks that pose this challenge include Imagenet, COCO, LAION, and many more.

\textbf{Condition 2.}
This condition must be satisfied for $C$ to perform tasks where the model must identify attributes associated with different objects in a scene. This could be towards scene understanding, vision question answering, or accurate image retrieval.
Specific datasets include CLEVR-bind, NCDataset-grayscale, VL-checklist, Sugarcrepe, ARO, and MMVP.

\textbf{Condition 3.}
This condition must be satisfied for $C$ to to perform tasks that require compositional image understanding. This could be for image captioning, text-guided image generation, spatial navigation, and more.  
Specific datasets include WhatsUP, Coco-spatial, MMVP, etc.

\textbf{Condition 4.}
This condition must be satisfied for $C$ to perform well on vision-language tasks that include prompts with negations. Note that this condition is a very relaxed interpretation of negation: 
Strictly semantically speaking,``not X" is a correct semantic pair with any image that does not have X, requiring a cosine similarity near 1. But the definition of the negation condition we impose does not require such granularity. Again, we seek to pose minimal constraints to identify whether there is some version of negation CLIP could attain under ideal settings.
Specific datasets that require negation understanding include NegBench, as well as other compositional benchmarks like VLM-checklist, Sugarcrepe, or MMVP.

\subsection{Lemma 2. Addendum}

\textbf{Derivation of $p$.} We define the perturbed vector:
\begin{equation}
    \textbf{i}(x_a) = (1 - \delta)\textbf{ i}(x) + p \textbf{t}(a), \| \textbf{i}(x_a) \|^2 = 1
\end{equation}
Expanding the norm, we have:
\begin{equation}
\begin{aligned}
     \| \textbf{i}(x_a) \|^2 &= \left\| (1 - \delta) \textbf{i}(x) + p \textbf{t}(a) \right\|^2 \\
     \| \textbf{i}(x_a) \|^2 &= (1 - \delta)^2 \| \textbf{i}(x) \|^2 + p^2 \| \textbf{t}(a) \|^2 \\&+ 2(1 - \delta) p \textbf{i}(x)\cdot \textbf{t}(a) 
\end{aligned}
\end{equation}
    
Since \( \textbf{i}(x) \) and \( \textbf{t}(a) \) are unit vectors, this simplifies to:
\begin{equation}
    \| \textbf{i}(x_a) \|^2 = (1 - \delta)^2 + p^2 + 2(1 - \delta) p \cos \theta
\end{equation}
where \( \cos \theta = \textbf{i}(x) \cdot \textbf{t}(a) \). For \( \textbf{i}(x_a) \) to be a unit vector, we set the right hand side to 1:
\begin{equation}
    \begin{split}
  (1 - \delta)^2 + p^2 + 2(1 - \delta) p \cos \theta = 1 \\   
  -2\delta + \delta^2 + p^2 + 2(1 - \delta) p \cos \theta = 0\\
    \end{split}
\end{equation}
Notice that this is now a quadratic equation in \( p \):
\begin{equation}
   p^2 + 2(1 - \delta) p \cos \theta - (2\delta - \delta^2) = 0. 
\end{equation}
Use the quadratic formula:
\begin{equation}
\resizebox{.85\linewidth}{!}{$
    p = \frac{-2(1 - \delta) \cos \theta \pm \sqrt{[2(1 - \delta) \cos \theta]^2 + 4(2\delta - \delta^2)}}{2} \\
$}
\end{equation}
We can simplify the above to find the correct value of \( p \) that ensures \( \textbf{i}(x_a) \) is a unit vector. 
\begin{equation}
  p = -(1 - \delta) \cos \theta \pm \frac{\sqrt{4(1 - \delta)^2 \cos^2 \theta + 8\delta - 4\delta^2}}{2}.  
\end{equation}
\bigskip

\noindent \textbf{Analysis for noise vectors.}
In Lemma 2 we derived 

\begin{equation}
\boxed{
\begin{aligned}
    \textbf{i}(x_a, y_b) &=\frac{(1 - \delta)( \textbf{i}(x)+\textbf{i}(y)) + p\textbf{t}(a) + q\textbf{t}(b)}{2} \\
    &= \textbf{i}(x_b, y_a)
    \end{aligned}
    }
\end{equation} 
using $\textbf{i}(x_a) = (1 - \delta) \textbf{i}(x) + p \textbf{t}(a)$.

Now we show that analytically, the inclusion of a noise vector $\boldsymbol{\epsilon}$ does not change the results. Specifically, we want to simulate 
 $$\textbf{i}(x_a) = \frac{(1 - \delta) \textbf{i}(x) + \delta \textbf{t}(a) + \boldsymbol{\epsilon}}{\left\|norm\right\|}$$ for some randomly sampled $\boldsymbol{\epsilon}$ and $\delta$. 

In Fig. \ref{fig:cossim_hist} we showcase the results of sampling $\boldsymbol{\epsilon}$ from a standard normal distribution, with varying weights. 
\begin{figure}[h]
    \centering

    \includegraphics[width=0.45\linewidth]{images/cosine_similarities_obj0-8_attr0-2_noise0-2.png}
    \includegraphics[width=0.45\linewidth]{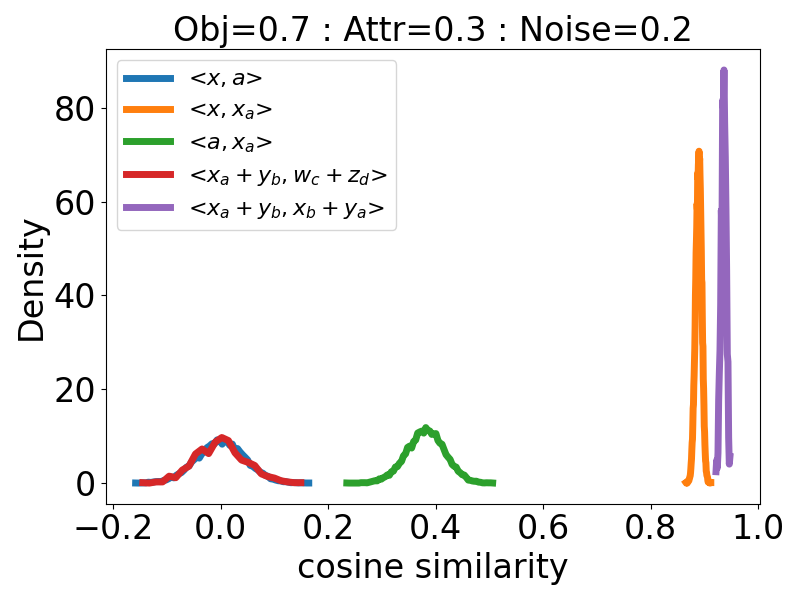}
    \includegraphics[width=0.45\linewidth]{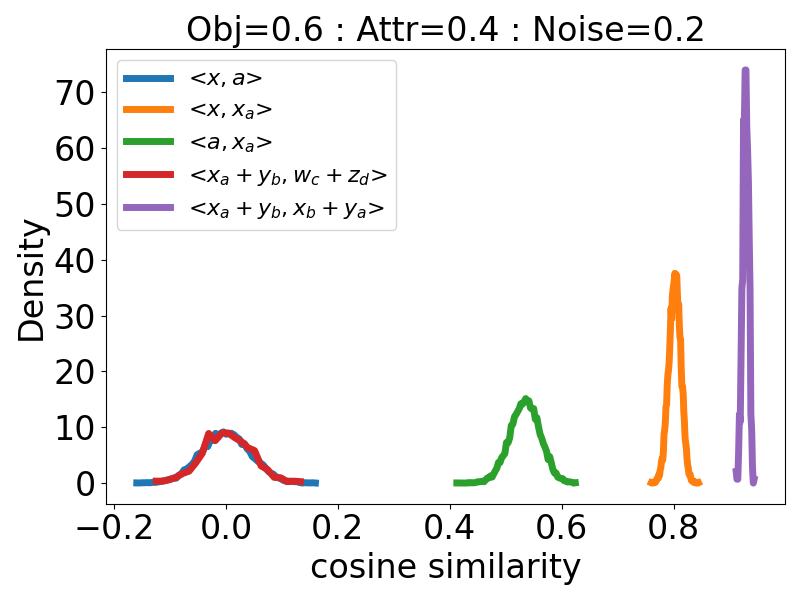}
    \includegraphics[width=0.45\linewidth]{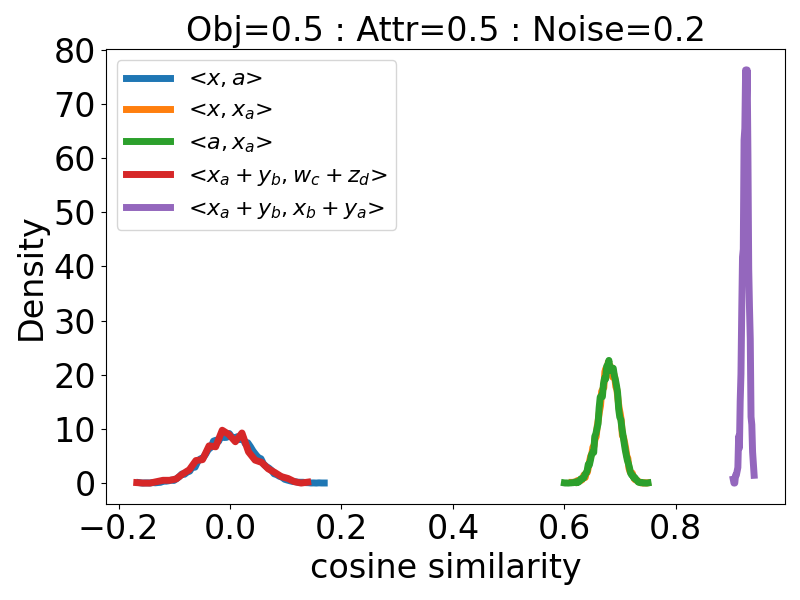}
    \caption{Titles of each subplot indicate the weights of the different components composing $\textbf{i}(x_a,y_b), \textbf{i}(x_b,y_a)$.
    Obj= indicates weight of the object concept embeddings $\textbf{i}(x), \textbf{i}(y)$, Attr= indicates weight of the attribute text embeddings $\textbf{t}(a), \textbf{t}(b)$, and Noise= indicates weight of the noise vector $\boldsymbol{\epsilon}$. In the legend, object concepts and attribute embeddings are denoted in shorthand: $x = \textbf{i}(x), a=\textbf{t}(a)$, and so on.}
    \label{fig:cossim_hist}
\end{figure}

We observe the following relations remain consistent:
\begin{itemize}
\item As expected of randomly initialized vectors in high dimensions: $\textbf{i}(x)\cdot\textbf{t}(a) \approx 0$
    \item For some unrelated object-attribute pairs, their image embeddings are roughly orthogonal as well: $\textbf{i}(w_c,z_d)\cdot\textbf{i}(x_a,y_b) \approx 0$
    \item The strong conclusion from Lemma 2 is approximately always true: $\textbf{i}(x_b,y_a)\cdot\textbf{i}(x_a,y_b) \approx 1$
\end{itemize}
regardless of $\textbf{i}(x_a)\cdot\textbf{t}(a), \textbf{i}(x)\cdot \textbf{i}(x_a)$.
As such, we see that even if there is noise in the superpositions described in Lemma 1 and 2, $C$ still cannot disambiguate between different pairings of the same two attributes and two objects.

\bigskip

\subsection{Contradiction for Condition 1 and 3}
Now we show Condition 3 cannot be met if Condition 1 is met. Below we will show two impossibility cases and prove them.

\begin{lemma}
\textbf{$C$ cannot accurately represent both the distance between spatial locations and relationships at the same time. }
\end{lemma}

\begin{proof}
We first derive $\textbf{i}(x,g^{<rel>}, y)$ for two antonymous $g^{<rel>}$ to satisfy Condition 1.2. Then we consider an example scenario with two objects, two spatial relationships, and two localizing terms. We will find that for three sample images, the cosine similarities between their embeddings and four textual clauses will have to contradict Condition 3 for some pairs.

For two concepts \(x\) and \(y\), their combined embedding is Eq.~(\ref{eq:sup}).
Now, if we want to express some compositional relationship $g^{<rel>} \in \mathbb{G}$ between x and y such that $\textbf{i}(x,g_1^{<rel>}, y) \neq \textbf{i}(x,g_2^{<rel>},y)$, we can write
\begin{equation}
    \begin{split}
   \textbf{i}(x,g^{<rel>}_1,y) 
        = (1-\delta)\,\textbf{i}(x,y) \;+\; \textbf{v}_1\\
        \textbf{i}(x,g^{<rel>}_2,y) 
        = (1-\delta)\,\textbf{i}(x,y) \;+\; \textbf{v}_2
    \end{split}
\end{equation}

where
\(\delta \ll 1\) and $\textbf{v}_1 \neq \textbf{v}_2$. Similar to Lemma~\ref{lemma_attbind}, $\mathbf{v}$ is a small location-specific component, as $\textbf{i}(x,g^{<rel>},y)$ must remain close to $\textbf{i}(x,y)$ per Condition 1.2.



 Let $g_L^{<rel>} = g_L$ be the relational concept whose equivalent mapping in $\mathbb{T}$ is ``\_ left of \_", and $g_R^{<rel>} = g_R$ ``\_ right of \_". Then we can write:
\[
   \textbf{i}(x, g_L, y) 
   \;=\; 
   (1-\delta)\,\textbf{i}(x,y) + \textbf{e}_{\perp,L}\]
   \[
   \textbf{i}(x, g_R, y)
   \;=\; 
   (1-\delta)\,\textbf{i}(x,y) + \textbf{e}_{\perp,R}
\]
where \(\textbf{e}_{\perp,L},\textbf{e}_{\perp,R}\) both lie in the orthogonal error subspace (of dimension \(N-1\)) and have fixed magnitude \(\sqrt{2\delta - \delta^2}\) such that $\textbf{i}(x, g_L, y) $ is a unit vector. 
Notice that we cannot use the intuition from Lemma 2 that $v$ must be composed of the textual component of $g_L$ - an image with a melon above a bed does not intuitively need to embed closely with the text embedding for ``above".


\begin{figure}[h]
   \centering
   \fbox{
   \includegraphics[width=0.7\linewidth]{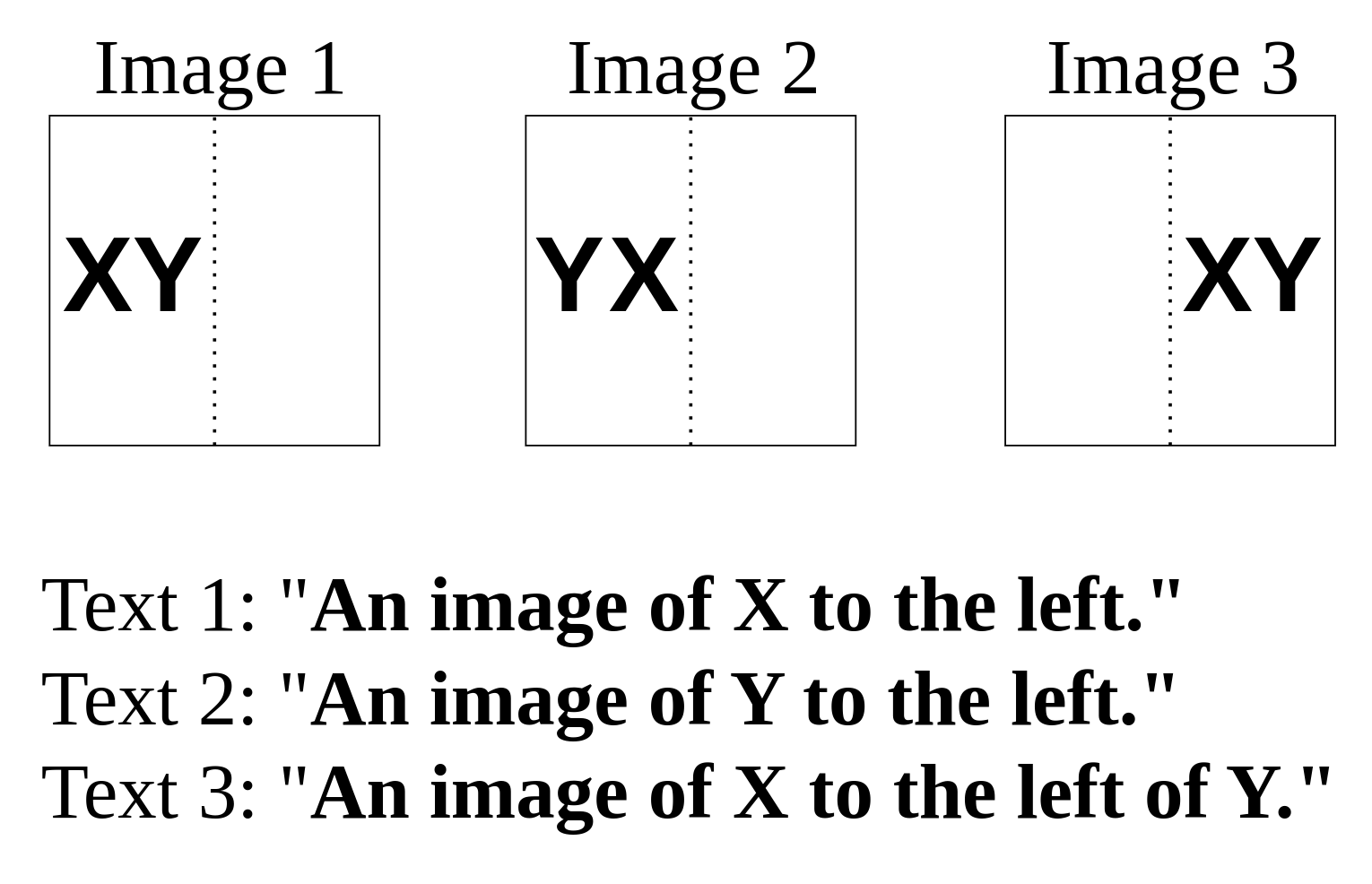}}
   \caption{Simple setup. Here the only relations that exist are ``\_ left \_", ``\_ right of \_", ``\_ to the left", and ``\_ to the right". 
   }
   \label{fig:xy-setup}
\end{figure}

Now we formulate a proof by contradiction.
Consider three images and three text prompts, as shown in Fig.~\ref{fig:xy-setup}. In addition to $g_L, g_R$, we also introduce $g_l^{<loc>} = g_l, g_r^{<loc>} = g_r$, to represent ``\_ to the left" and ``\_ to the right" in $\mathbb{T}$, respectively. These are the only four compositional concepts we consider for simplicity, but conclusions will generalize.

We denote $(1-\delta)\,\textbf{i}(x,y) = \textbf{e}_{||}$, the three images as $\textbf{i}( \text{image}_k ) = \textbf{i}_k$ and $\textbf{t}(\text{text}_j) = \textbf{t}_j$, for $k,j \in {1,2,3}$. 
$\textbf{i}_k = \textbf{e}_{||} + \textbf{e}_{\perp,k}$ where $\textbf{e}_{||}$ is the shared parallel component for all $\textbf{i}_k$. 

Consider some images $\textbf{i}_w, w>3$, where $x,y$ do not have $g_L, g_R, g_l$ or $, g_r$ as a location or relationship. For example $\textbf{i}_4$ could be an image of $x$ on top of $y$ in the center of the frame. In this case,
$$\textbf{i}_k \cdot \textbf{i}_w = (\textbf{e}_{||} + \textbf{e}_{\perp,k})\cdot(\textbf{e}_{||} + \textbf{e}_{\perp,w}) = |\textbf{e}_{||}|, k \in (1,2,3)$$ as these image pairs share no compositional aspects in common. Since $\textbf{i}_1,\textbf{i}_2$ share the same localization for the objects present, satisfaction of Condition 3.2 requires that $\textbf{i}_1\cdot \textbf{i}_2 > |\textbf{e}_{||}|$.
Similarly, as $\textbf{i}_1,\textbf{i}_2$ share the same relationship between the objects present, satisfaction of Condition 3.3 requires that $\textbf{i}_1 \cdot \textbf{i}_3 > |\textbf{e}_{||}|$. 

This could only happen if embeddings of images 1,2 and i,3 have perpendicular components that are partially parallel. 
In other words, \textbf{in order to satisfy Conditions 3.2 and 3.3, 
there must exist a $C$ in which $\textbf{e}_{\perp,1}\cdot \textbf{e}_{\perp,2} > 0$ and $\textbf{e}_{\perp,1} \cdot \textbf{e}_{\perp,3} > 0$}.


For each image in Fig.~\ref{fig:xy-setup}, the embeddings must optimize the following local similarities:
\begin{equation}
\begin{split}
    \textbf{i}_1 = \operatorname*{argmax}_{\textbf{i}_1}(\textbf{i}_1 \cdot \textbf{t}_1 + \textbf{i}_1 \cdot \textbf{t}_2 + \textbf{i}_1 \cdot \textbf{t}_3) \\
    \textbf{i}_2 = \operatorname*{argmax}_{\textbf{i}_2}(\textbf{i}_2 \cdot \textbf{t}_1 + \textbf{i}_2\cdot \textbf{t}_2 - \textbf{i}_2\cdot\textbf{t}_3) \\
    \textbf{i}_3 = \operatorname*{argmax}_{\textbf{i}_3}(-\textbf{i}_3\cdot \textbf{t}_1 - \textbf{i}_3\cdot \textbf{t}_2 + \textbf{i}_3\cdot\textbf{t}_3) 
\label{eq:i_def}
\end{split}
\end{equation}
 subject to $\left\|\textbf{i}_k\right\| = 1$. This allows us to solve for $\textbf{e}_{\perp,k}$s. For $k=1$: 
\begin{equation}
    \begin{split}
\textbf{e}_{\perp,1} = \operatorname*{argmax}_{\textbf{e}_{\perp,1}}((\textbf{e}_{||} + \textbf{e}_{\perp,1}) \cdot \textbf{t}_1 + ((\textbf{e}_{||} + \textbf{e}_{\perp,1}) \cdot \textbf{t}_2 + \\((\textbf{e}_{||} + \textbf{e}_{\perp,1})\cdot\textbf{t}_3)
    \end{split}
\end{equation}
subject to $\left\|\textbf{e}_{\perp,1}\right\| = \sqrt{2\delta - \delta^2}$. $\textbf{e}_{||}$ is fixed so this becomes:
\begin{equation}
\begin{aligned}
\textbf{e}_{\perp,1} &= \operatorname*{argmax}_{\textbf{e}_{\perp,1}} (\textbf{e}_{||} \cdot (\textbf{t}_1 + \textbf{t}_2 + \textbf{t}_3) + \textbf{e}_{\perp,1}\cdot(\textbf{t}_1 + \textbf{t}_2 + \textbf{t}_3))
\\
&= (\textbf{t}_1 + \textbf{t}_2 + \textbf{t}_3)\cdot \sqrt{2\delta - \delta^2}
\label{eq:e1}
    \end{aligned}
\end{equation}

Similarly, we get:
\begin{equation}
\begin{split}
    \textbf{e}_{\perp,2} = (\textbf{t}_1 + \textbf{t}_2 - \textbf{t}_3)\cdot \sqrt{2\delta - \delta^2} \\ \textbf{e}_{\perp,3} =(-\textbf{t}_1 - \textbf{t}_2 + \textbf{t}_3) \cdot \sqrt{2\delta - \delta^2}\label{eq:e2}
\end{split}
\end{equation}

Now taking the dot products, we have:
\begin{equation}
\begin{split}
        \textbf{e}_{\perp,1} \cdot \textbf{e}_{\perp,2} = (|\textbf{t}_1| + |\textbf{t}_2| - |\textbf{t}_3|) \cdot 2\delta = 2\delta \\
    \textbf{e}_{\perp,1} \cdot \textbf{e}_{\perp,3} = (-|\textbf{t}_1| - |\textbf{t}_2| + |\textbf{t}_3|) \cdot 2\delta = -2\delta\\
    \textbf{e}_{\perp,2} \cdot \textbf{e}_{\perp,3} = (-|\textbf{t}_1| -|\textbf{t}_2| - |\textbf{t}_3|) \cdot 2\delta = -6\delta
\end{split}
\end{equation}

Practically speaking, it is possible by adding more representative samples in the training dataset to change the weights of $\textbf{t}_j$s. That is, for some $\beta_1 + \beta_2 + \beta_3 = 3$, Eqs.~(\ref{eq:e1},\ref{eq:e2}) could be reformulated as:
\begin{equation}
\begin{split}
\textbf{e}_{\perp,1} = (\beta_1\textbf{t}_1 + \beta_2\textbf{t}_2 + \beta_3\textbf{t}_3)\cdot \sqrt{2\delta - \delta^2} \\
\textbf{e}_{\perp,2} = (\beta_1\textbf{t}_1 + \beta_2\textbf{t}_2 - \beta_3\textbf{t}_3)\cdot \sqrt{2\delta - \delta^2} \\
\textbf{e}_{\perp,3} =(-\beta_1\textbf{t}_1 - \beta_2\textbf{t}_2 + \beta_3\textbf{t}_3) \cdot \sqrt{2\delta - \delta^2}\label{eq:e3}
\end{split}
\end{equation}

The dot products then become:
\begin{equation}
\boxed{
\begin{aligned}
\textbf{e}_{\perp,1} \cdot \textbf{e}_{\perp,2} 
   &= 2\delta\,\bigl(\beta_1^2 + \beta_2^2 - \beta_3^2\bigr)\\
     \textbf{e}_{\perp,1} \cdot \textbf{e}_{\perp,3} 
   &= 2\delta\,\bigl(-\beta_1^2 \;-\;\beta^2 \;+\;\beta_3^2\bigr)\\
   \textbf{e}_{\perp,2} \cdot \textbf{e}_{\perp,3} 
   &=\;-2\delta\,(\beta_1^2 + \beta_2^2 + \beta_3^2)
   \label{eq:beta_inv}
\end{aligned}
}
\end{equation}

Note that regardless of the reweighting, $\textbf{e}_{\perp,1} \cdot \textbf{e}_{\perp,2} = - \textbf{e}_{\perp,1} \cdot \textbf{e}_{\perp,3}$. This directly negates our previous observation that an ideal $C$ must satisfy $\textbf{e}_{\perp,1}\cdot \textbf{e}_{\perp,2} > 0$ and $\textbf{e}_{\perp,1} \cdot \textbf{e}_{\perp,3} > 0$. As such, there exists no $C$ which sufficiently represents both relational and objective space in the image embeddings.

\end{proof}

\begin{lemma}
\textbf{    $C$ cannot accurately represent compositional concepts of different hierarchy.}
\end{lemma}
\begin{proof}
Here we will show that general prepositions are erroneously closer to all unrelated prepositions.

Some prepositions are more general than others.
For example, $g_B^{<rel>} = g_B$ where $f_{G,T}(g_B) = \text{"\_ beside \_"}$ semantically includes both $g_L$ and $g_R$. The ideal placement for $\textbf{t}(x,g_B,y)$ should locally optimize for the following similarities:

\begin{equation}
    \begin{split}
    \textbf{t}(x,g_B,y) = \operatorname*{argmax}_{\textbf{t}(x,g_B,y)}\Big[
\textbf{t}(x,g_L,y) \cdot \textbf{t}(x,g_B,y)  \\
+ \textbf{t}(x,g_R,y) \cdot \textbf{t}(x,g_B,y) \\-\sum_{z=1}^{|\mathbb{G}\backslash\{g_L,g_R,g_B\}|} \textbf{t}(x, g_z, y) \cdot \textbf{t}(x, g_B,y)\Big]
    \end{split}
\end{equation}
We know from Lemma 1 that this means $\textbf{t}(x,g_B,y)$ should be a weighted superposition of $\textbf{t}(x,g_L,y)$ and $\textbf{t}(x,g_R,y)$.
Since we know $\textbf{e}_L = -\textbf{e}_R$, we can write:
$$\textbf{t}(x,g_L,y) = (1-\delta) \cdot \textbf{t}(x,y) + \textbf{e}_L$$
$$\textbf{t}(x,g_R,y) = (1-\delta) \cdot \textbf{t}(x,y) - \textbf{e}_L$$

Then the superposition of the two becomes $\textbf{t}(x,y)$. But this is a semantically erroneous placement for $\textbf{t}(x,g_B,y)$, as it will incorrectly be closer to any other instance of $x$ and $y$ co-appearing in a scene than the average $\textbf{t}(x,g^{<rel>},y)$.
For example, for $g_A^{<rel>} = g_A$ where $f_{G,T}(g_A) = \text{"\_ above \_"}$, $\textbf{t}(x,g_B,y)\cdot\textbf{i}(x,g_A,y) >\textbf{t}(x,g_L,y)\cdot \textbf{i}(x,g_A,y)$ even though the two captions are equally inapplicable.





\end{proof}

\subsection{Contradiction for Condition 1 and 4}
Now we show Condition 4 cannot be met if Condition 1 is met. Before we move to the proof, we first discuss in greater detail the possible arrangements of object concepts in $C$.

\noindent\textbf{Orthogonality.} One straightforward intuition is that since $C$ is a high dimensional space, for any two random concepts $x^1, x^2 \in \mathbb{V}$, they should be approximately orthogonal~\cite{6771362}:
\begin{equation}
    \textbf{t}(x^j)\cdot \textbf{t}(x^k) \approx 0 \quad \forall  j \neq k
\end{equation}

This makes it trivial to derive that $\textbf{t}(\neg x)\cdot \textbf{t}(y) \approx \textbf{t}(x)\cdot\textbf{t}(y) = 0$, violating Condition 4.2. We clarify in Lemma 5 why that is. In order to be more rigorous, we show that negation cannot be achieved even under strong idealistic conditions, where $\textbf{t}(x)$ are uniformly distributed and \textit{distinguishable} from one another. This requires perfect isotropy of $M$ concepts. 

\noindent\textbf{Isotropy.} Starting with $|\mathbb{V}|$ = $M \leq N$ concepts, we determine the ideal distribution for $\textbf{t}(x)$ $\forall x \in \mathbb{V}$. As in Lemma 1, we denote distinct concepts as: $\textbf{t}(x^1), \textbf{t}(x^2), ...$. Since all $x \in \mathbb{V}$ are mutually exclusive in semantics by definition, the distance between any two arbitrary concepts $x^1, x^2$ should be comparable to the distance between $x^1, x^k$ for some $x^k \in \mathbb{V}\backslash \{x^2\}$.
Then $C$ must minimize the variance among the cosine similarities of all pairs of concepts:
\begin{equation}
\begin{split} 
\operatorname*{min}_{\textbf{t}(x^j) \cdot \textbf{t}(x^k)} \quad \sum_{j=1}^{M} \sum_{k>j}^{M} \left( \textbf{t}(x^j) \cdot \textbf{t}(x^k) - \bar{s} \right)^2 \\
\text{s. t. } \|\textbf{t}(x^k)\| = 1, \quad \forall x^k \in \mathbb{V}
\end{split}
\label{eq1}
\end{equation}
where $\bar{s}$ is the mean cosine similarity:
\begin{equation}
\bar{s} = \frac{1}{\binom{M}{2}} \sum_{j=1}^{M} \sum_{k>j}^{M} \textbf{t}(x^j) \cdot \textbf{t}(x^k)
\end{equation}

\noindent Let $\textbf{t}(x^j)\cdot \textbf{t}(x^k) = s_{jk}$. Then the objective simplifies to:
\begin{equation}
\operatorname*{min}_{s_{jk}}\sum_{j=1}^{M} \sum_{k>j}^{M} \left( s_{jk}^2 - \bar{s}^2 \right) = \operatorname*{min}_{s_{jk}} \sum_{j=1}^{M} \sum_{k>j}^{M} (s_{jk})^2 - \bar{s}^2 \binom{M}{2}
\end{equation}

\noindent Differentiate the first and second terms with respect to $s_{jk}$:
\begin{equation}
\frac{\partial}{\partial s_{jk}} \sum_{j=1}^{M} \sum_{k>j}^{M} s_{jk}^2 = 2 s_{jk},
\end{equation}
\begin{equation}
\frac{\partial}{\partial s_{jk}} \left( \bar{s}^2 \binom{M}{2} \right) =  2\binom{M}{2} \bar{s} \frac{\partial \bar{s}}{\partial s_{jk}} = \frac{1}{\binom{N}{2}}2\binom{M}{2} \bar{s} = 2\bar{s}
\end{equation}

\noindent This means that the optimum of Eq.~(\ref{eq1}) is reached when 
\begin{equation}
\begin{split}
        2\bar{s} -2s_{jk}=0 
    \quad\forall j,k \in (1,M), j\neq k
\end{split}
\end{equation}

In other words, the cosine similarities between any two object concepts must be the same as the average. That requires for all $\textbf{t}(x^k)$ to be isotropically distributed in $\mathbb{R}^N$. The optimal arrangement of all $\textbf{t}(x^k)$ is then a $M$-1 dimensional regular simplex, which is a structure of $M$ equiangular unit vectors in $\mathbb{R}^M$. Then we have that the cosine similarity of two random vectors in $C$ is:
\begin{equation}
 \boxed{
 \begin{gathered}
 $$\textbf{t}(x^j)\cdot \textbf{t}(x^k) = - \frac{1}{M-1} \quad \forall  j \neq k$$ 
 \end{gathered}
 }
\end{equation}
as all vector pairs in a regular simplex have angles $\arccos{(-\frac{1}{M-1})}$~\cite{Parks2002-bi}. 

\begin{lemma}
\textbf{Even under isotropic concept distribution, $C$ cannot accurately represent negation.}
\end{lemma}

\begin{proof}

We first derive what $\textbf{t}(x), \textbf{t}(\neg x)$ must be to respect Conditions 4.1 and 4.2. Then we see that this derivation contradicts Condition 4.3.

For $C$ to ideally represent negation, the following must be true:
\begin{equation}
    \begin{aligned}
    \textbf{t}(\text{"$\neg$ $x$"}) \cdot \textbf{i}(x) &< \textbf{t}(\text{"$\neg$ $x$"}) \cdot \textbf{i}(v)\quad (1)\\
    \textbf{t}(\text{"$\neg$ $x$"}) \cdot \textbf{i}(v) &> \textbf{t}(\text{"$x$"}) \cdot \textbf{i}(v)\quad (2)\\ \textbf{t}(\text{"$x$"}) \cdot \textbf{t}(\text{"$y$"}) &< 
    \textbf{t}(\text{"$\neg$$x$"}) \cdot \textbf{t}(\text{"$\neg$$y$"}) \quad (3)
    \\\text{ for all } v \in \mathbb{V} \backslash \{x\}
    \end{aligned}
    \label{eq:neg}
\end{equation}

To achieve Eq.~(\ref{eq:neg}.1), we solve for:
\begin{equation}
    \begin{aligned}
    \textbf{t}(\neg x) &= \operatorname*{argmax}_{\textbf{t}(\neg x)}\Big[ \sum_{v=1}^{| \mathbb{V} \backslash \{x\}|} \textbf{t}(\neg x) \cdot \textbf{i}(v) - 
\textbf{t}(\neg x) \cdot \textbf{i}(x) \Big]\\
    &= \operatorname*{argmax}_{\textbf{t}(\neg x)}\Big[  \textbf{t}(\neg x) \cdot \Big(\sum_{v=1}^{|\mathbb{V}|}\textbf{i}(v) - \textbf{i}(x) -  \textbf{i}(x)\Big) \Big]
    \end{aligned}
\end{equation}

Here, $\sum_{v=1}^{|\mathbb{V}|}\textbf{i}(v)$ is the sum of all vectors in a regular simplex, which is 0. As such, we find that:

\begin{equation}
    \textbf{t}(\neg x)^*  = -\textbf{i}(x)
\end{equation}


Given this, we can see how if $\textbf{t}(y), \textbf{t}(x)$ were orthogonal, then $\textbf{t}(x)\cdot\textbf{t}(y) = \textbf{t}(\neg x)\cdot \textbf{t}(y)$ is trivial to show. In high dimensions, this orthogonality is always approximately true, violating Condition 4.2.

But let's consider the more rigorous case where $\textbf{t}(x), \textbf{t}(y)$ are not orthogonal but rather isotropically distributed in N dimensional space. 
Recall that for two vectors in an $M-1$ dimensional simplex, $\textbf{t}(x^j) \cdot \textbf{t}(x^k) =-\frac{1}{M-1}\ \forall j \neq k$. With the above solution, we now have that:
\begin{equation}
\textbf{t}(\neg x^j) \cdot \textbf{t}(x^k) = \frac{1}{M-1} > \textbf{t}(x^j) \cdot \textbf{t}(x^k)
\end{equation}

\noindent which satisfies condition 4.2. But 4.3 fails due to the following:

\begin{equation}
\boxed{
\begin{aligned}
    \textbf{t}(\neg x^j) \cdot \textbf{t}(\neg x^k)   &= \textbf{t}(x^j) \cdot \textbf{t}(x^k)  = -\frac{1}{M-1}  \\
    \textbf{t}(\neg x^j)\cdot \textbf{t}(x^k) &> \textbf{t}(\neg x^j) \cdot \textbf{t}(\neg x^k)
    \end{aligned}
    }
\end{equation}

Let's say $x^j=$ ``chair" and $x^k=$ ``penguin". The first erroneous semantic relationship that emerges is that the distance between ``chair" and ``penguin", which are fully contradictory statements, will be equivalent to the distance between ``Not chair" and ``Not penguin". For the latter two prompts there exist a lot of images that would be a true match for both, whereas for the first prompt there exists only one. 

The second erroneous conclusion is that the cosine similarity between ``Not chair" and ``penguin" is  greater than the cosine similarity between ``Not chair" and ``Not penguin". This is semantically incorrect for the same reason as above.





\end{proof}


 


 






\section{Open Vocabulary Experimental Details}

We train two model types, one on 5k COCO images from the training split for 12 epochs, and another on 10k COCO images for 6 epochs. The images each have a hard positive and negative, where the latter is the same as the former save for two nouns being swapped.
We choose this particular type of intervention as CLIP-like VLM performance across the board was lowest for this category of hard negatives in Sugarcrepe \cite{hsieh_sugarcrepe_2023}. We evaluate this paradigm on the swap-object split of Sugarcrepe and the VG-spatial split of VL-Checklist~\cite{zhao_vl-checklist_2023} against naive and finetuned CLIP. We choose this particular split of VL-Checklist because we find that all other splits (Objects, Attributes, or Relation-Action) do not introduce new functional words and are thus not applicable for testing dynamic FR updates. 

All LLM queries were made to OpenAI's gpt4o-mini model, with a temperature of 0.7. At the time of evaluation, this was the most affordable model on the API at: \$ 0.150 / 1M input tokens and \$0.600 / 1M output tokens.

\subsection{LLM System Prompts}

\begin{table}[t]
\centering
\begin{tcolorbox}[colback=gray!5.5!white,
   colframe=black!65!black, boxrule=0.3pt, rounded corners, arc=2mm]
\textbf{Original Prompts:}\vspace{0.5em}\\
{\small['A baby elephant stands next to its mother.', 'A desktop computer sitting on top of a wooden desk.', 'A mother stands next to its baby elephant.', 'A wooden computer sitting on top of a desk.']}\vspace{0.5em}\\
\textbf{Original Lookup Table:}\vspace{0.5em}\\
{\small{[ 'above', 'below', 'many', 'no', 'small', 'big', 'not', 'without', 'left', 'right', 'absent', 'but', 'large' ]}\vspace{0.5em}}\\ 
\textbf{Simplified Prompts:}\vspace{0.5em}\\
{\small[ 'Baby elephant next to mother', 'Desktop computer on wooden desk', 'Mother next to baby elephant', 'Wooden computer on desk' ]}\vspace{0.5em}\\
\textbf{New Lookup Table:}\vspace{0.5em}\\
{\small['above', 'below', 'many', 'no', 'small', 'big', 'not', 'without', 'left', 'right', 'absent', 'but', 'large', 'near', 'on']}\vspace{0.5em}\\
\textbf{New Prompts:}\vspace{0.5em}\\
{\small['Baby elephant $\langle$NEAR$\rangle$ mother', 'Desktop computer $\langle$ON$\rangle$ wooden desk', 'Mother $\langle$NEAR$\rangle$ baby elephant', 'Wooden computer $\langle$ON$\rangle$ desk']}\vspace{0.5em}\\
\end{tcolorbox}
\caption{Examples of LLM-in-the-loop natural language simplification and FR extraction.}
\label{tab:prompt_example}
\vspace{-1.2em}
\end{table}

\begin{table*}[t]
\centering
\begin{tcolorbox}[colback=gray!5.5!white,
   colframe=black!65!black, boxrule=0.3pt, rounded corners, arc=2mm]
\textbf{System prompts:}\vspace{0.5em}\\
Given a list of sentences, reformat each sentence to the simplest phrases that would distinguish it from the other examples. Here are some formatting rules to follow:\vspace{0.5em}\\

1. If a sentence contains OBJECTS and ATTRIBUTES which belong to that object, the ATTRIBUTE must always come first. For example, given the sentence ``A dog which is purple", reformat it to ``A purple dog". \vspace{0.5em}\\

2. If a sentence contains NEGATION, the NEGATING TERM always comes before the OBJECT clause. For example, given ``This image contains a chicken but a butterfly is absent", reformat it to ``Chicken but no butterfly". \vspace{0.5em}\\

3. If a sentence contains PREPOSITIONs, try your best to make sure that the OBJECTS the PREPOSITION is describing are immediately before and after the PREPOSITION. For example, given ``A bug which is flying much farther up from the bench", reformat it to ``A flying bug above a bench".\vspace{0.5em}\\

4. Whenever possible, reformat VERBs to be ATTRIBUTES. For example, given ``A dog dancing while his owner is jumping", reformat it to ``Dancing dog and jumping owner".\vspace{0.5em}\\

5. If two sentences are very close to each other, reduce them down to the salient components. For example, given the sentences ["Butterflies in the clouds, a cat squatting looking up at it, and a man standing behind the cat watching it, on the grass with a tree.", ``Butterflies in the clouds, a cat squatting looking up at it, and a man sitting behind the cat watching it, on the grass with a tree."], return: {0: ``A man standing", 1: ``A man sitting"}. (Of course, if there are other sentences in the list that are similar, you may want to keep more details so that the sentences are still distinguishable.)\vspace{0.5em}\\

Here are some more general examples. If given the following sentences: ["A desktop computer sitting on top of a gray oak table lights up the room" ,"A gray oak computer sitting on top of a desktop table lights up the room", ``A kitchen has metal cabinets and black countertops with shiny lights on top.", ``A kitchen has black cabinets and metal countertops with shiny lights on top."], return:
{0: ``Desktop computer top of gray oak table", 1: ``Gray oak computer top of desktop table", 2: ``Metal cabinets and black countertops", 3: ``Black cabinets and metal countertops"}.\vspace{0.5em}\\

As a rule, be AS CONCISE AS POSSIBLE. If any information is repeated and unnecessary to keep in order to distinguish that text prompt from the others, discard it.\vspace{0.5em}\\

Now, reformat the following list of sentences and return the JSON output. Do not include anything other than the JSON array in your answer.
\end{tcolorbox}
\caption{Prompt template for simplifying natural language prompts.}
\label{tab:prompt_simplify_language}
\vspace{-1.2em}
\end{table*}

\begin{table*}[t]
\centering
\begin{tcolorbox}[colback=gray!5.5!white,
   colframe=black!65!black, boxrule=0.3pt, rounded corners, arc=2mm]
\textbf{System prompts:}\vspace{0.5em}\\
You are given: A LOOKUP LIST of functional words (e.g., {\small["ABOVE", ``BELOW", ``INSIDE OF", ``MANY", ``SMALL", ``NO"]}). A list of SENTENCES to process. Definitions: Functional words include: (a) Prepositions (e.g., ABOVE, BELOW, INSIDE OF, ON, IN, NEAR) or their synonyms. (b) Size/shape terms (e.g., SMALL, BIG). (c) Numerical terms (e.g., ONE, TWO, THREE, etc.). If a number is greater than 5 (e.g., SEVEN, 100), replace it with ``MANY". (d) Negatory terms (e.g., NO, WITHOUT). Non-functional words: Do not include verbs, adjectives, or any nouns unrelated to the functional categories above. Examples of non-functional words include ``jumping", ``sleeping", ``cat", ``man", etc. These should not be added to the LOOKUP LIST, even if they appear in the sentences. Even if there are two sentences that are very similar, do not try to distinguish them by adding these verbs, adjectives, or nouns to the LOOKUP LIST. Any form of a verb, including present participles, may not go in the LOOKUP LIST, no matter how frequently it appears in the sentences.\vspace{0.5em}\\
Rules: For each sentence: Identify any functional words or synonyms (including numbers). If a functional word or one of its synonyms (by meaning) appears in the sentence and is already in the LOOKUP LIST, replace it in the sentence with the LOOKUP LIST key, surrounded by angle brackets (e.g., ``$\langle$ ABOVE$\rangle$ ``). If that functional word is not in the LOOKUP LIST (and it is truly functional by the above  definition), add it to the LOOKUP LIST, then replace its appearance with that new all-caps key in angle brackets. Do not add duplicates to the LOOKUP LIST. Do not add verbs, adjectives, or any non-functional words to the LOOKUP LIST. Replace numbers greater than 5 with ``MANY" (add ``MANY" to the list if not already present). After processing all sentences, output exactly one JSON array containing two sub-arrays: The first sub-array: the UPDATED LOOKUP LIST (only functional words, no duplicates). The second sub-array: the FINAL TRANSFORMED SENTENCES (with functional words surrounded by $\langle$  $\rangle$ ). Not every sentence needs functional words. Provide no additional commentary or text besides this JSON structure. Example of the required output format: {\small[ [ ``ABOVE", ``INSIDE OF", ``MANY" , ``RIGHT OF"], [ ``A bird $\langle$ ABOVE$\rangle$  a tree", ``Fifteen dogs is $\langle$ MANY$\rangle$  dogs" , ``A sitting chicken is $\langle$ INSIDE OF$\rangle$  a house"] ] Example of wrong output: [ [ ``ABOVE", ``INSIDE OF", ``MANY" , ``RIGHT OF", ``SITTING", ``DOGS", ``HOUSE"], [ ``A bird $\langle$ ABOVE$\rangle$  a tree", ``Fifteen dogs is $\langle$ MANY$\rangle$  dogs" , ``A sitting chicken is $\langle$ INSIDE OF$\rangle$  a house"] }]\vspace{0.5em}\\

Before you return the output, CHECK THAT THE LOOKUP LIST ONLY CONTAINS FUNCTIONAL WORDS.
\end{tcolorbox}
\caption{Prompt template for extracting functional words.}
\label{tab:prompt_functional_words}
\vspace{-1.2em}
\end{table*}

\begin{table*}
\centering
\begin{tcolorbox}
[colback=gray!5.5!white,
   colframe=black!65!black, boxrule=0.3pt, rounded corners, arc=2mm]
\textbf{Swapping Objects:}
   
Given an input sentence describing a scene, your task is to first locate two swappable noun phrases in the sentence, and then swap them to make a new sentence.
The new sentence must meet the following three requirements:
1. The new sentence must be describing a different scene from the input sentence.
2. The new sentence must be fluent and grammatically correct.
3. The new sentence must make logical sense.
\vspace{0.5em}\\
To complete the task, you should:
1. Answer the question of whether generating such a new sentence is possible using ``Yes" or ``No".
2. Output the swappable noun phrases.
3. Swap them to make a new sentence.
\vspace{0.5em}\\
Please produce a **single JSON array** (no extra text or explanation) for each input sentence.
If there are K input sentences, return a list with K JSON objects separated by commas.
Each element in the array must be a JSON object with the following structure:
\{
  ``possible": ``$\langle$Yes or No$\rangle$",
  ``swappable-noun-phrases": ["$\langle$NP1$\rangle$", ``$\langle$NP2$\rangle$"],
  ``swapped-sentence": ``$\langle$the swapped sentence$\rangle$"
\}
\vspace{0.5em}\\
Example JSON output for the original sentence: ``A cat resting on a laptop next to a person."
\{
  ``possible": ``Yes",
  ``swappable-noun-phrases": ["laptop", ``person"],
  ``swapped-sentence": ``A cat resting on a person next to a laptop." \}
\end{tcolorbox}
\caption{Prompt template for creating captions for swapped objects. This is a minorly changed version from \cite{hsieh_sugarcrepe_2023}.}
\label{tab:prompt_swap_obj}
\end{table*}

\end{document}